%% file: main.tex
\documentclass[10pt,twocolumn,letterpaper]{article}

\usepackage[pagenumbers]{cvpr} 

\input{preamble}

\definecolor{cvprblue}{rgb}{0.21,0.49,0.74}
\usepackage[pagebackref,breaklinks,colorlinks,allcolors=cvprblue]{hyperref}
\usepackage{multirow}
\usepackage{algorithm}
\RequirePackage{algpseudocode}
\RequirePackage{natbib}
\RequirePackage{mathtools}
\usepackage{cuted}

\def\confName{CVPR}
\def\confYear{2026}

\title{OrthoFuse: Training-free Riemannian Fusion of Orthogonal Style-Concept Adapters for Diffusion Models}

\author{Ali Aliev \footnotemark[1] \textsuperscript{,1,}\footnotemark[2]
\quad Kamil Garifullin \footnotemark[1] \textsuperscript{,1,2,3,}\footnotemark[2] 
\quad Nikolay Yudin \textsuperscript{1} 
\quad  Vera Soboleva \textsuperscript{1,2,3} \\
\quad Alexander Molozhavenko \textsuperscript{1} 
\quad Ivan Oseledets \textsuperscript{3} 
\quad  Aibek Alanov \textsuperscript{1,2,3}
\quad  Maxim Rakhuba \textsuperscript{1} \\
${^1}$HSE University, 
${^2}$FusionBrain Lab, ${^3}$AXXX 
}

\begin{document}

\twocolumn[{
\renewcommand\twocolumn[1][]{#1}
\maketitle
\begin{center}
    \includegraphics[width=0.96\linewidth]{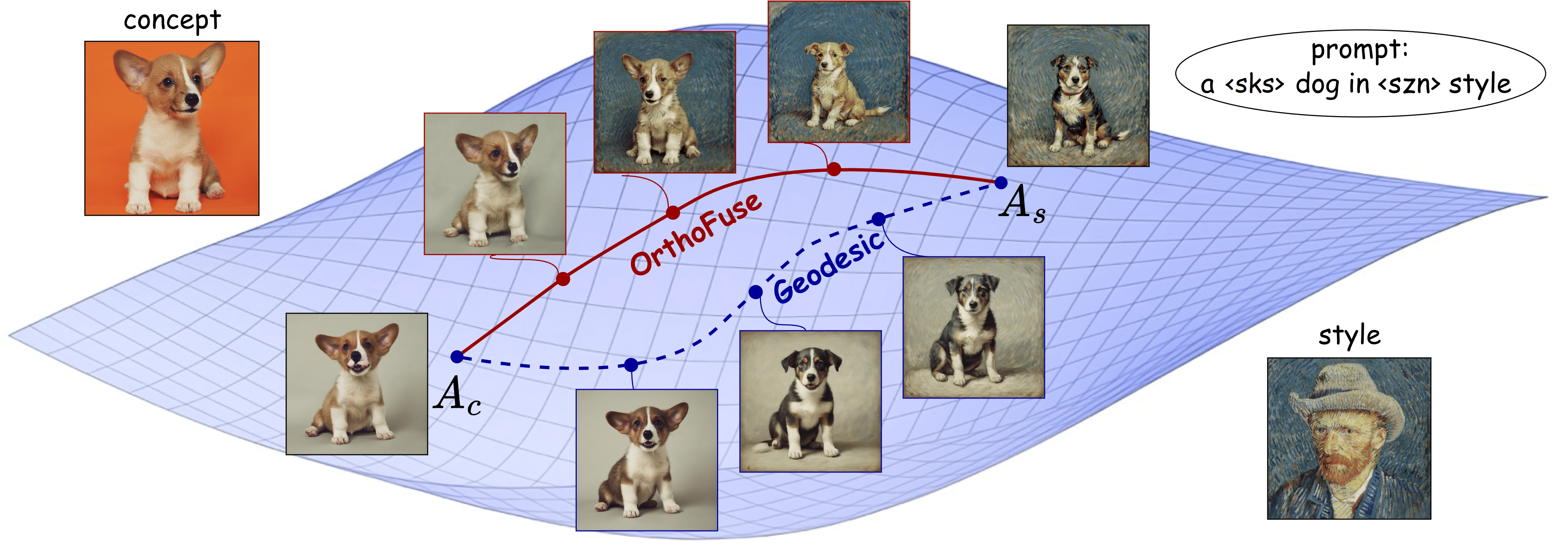}
 
    \captionsetup{type=figure}
    \caption{%
        {Overview of the proposed method \textbf{OrthoFuse}. By considering \lpr~orthogonal adapters as elements of the manifold, we are able to draw curves between them to fuse adapters into one with common features of object and style in particular proportion. To facilitate the generation quality, we analyze the spectrum of the orthogonal blocks inside \lpr ~representation and propose a specific curve on the manifold, aiming to preserve the distribution of blocks' eigenvalues.} 
    }
    \label{fig:visual_abs}
\end{center}
}]
\footnotetext[1]{Equal contribution.}
\footnotetext[2]{Correspondence: alievali0278@gmail.com, garifullin.k@hse.ru}
\input{sec/0_abstract}    
\input{sec/1_intro}
\input{sec/2_related_work}
\input{sec/3_preliminaries}

\input{sec/4_method}
\input{sec/5_experiments}
\input{sec/6_conclusion}
{
    \small
    \bibliographystyle{ieeenat_fullname}
    \bibliography{biblio}
}
\input{sec/X_suppl}

\end{document}

%% file: preamble.tex
\def\lpr{$\mathcal{GS}$}

\def\diag{\operatorname{diag}}

\newtheorem{definition}{Definition}[section]
\newtheorem{theorem}{Theorem}

\newtheorem{remark}{Remark}

\newtheorem{lemma}{Lemma}

\newtheorem{prop}{Proposition}

\usepackage{pgfplots}
\pgfplotsset{compat=1.18}

\usepackage{tikz}
\usetikzlibrary{3d, decorations.markings}

\tikzset{
  decoA/.style={
    decoration={
      markings,
      mark=at position 0.2 with {
        \coordinate (P1) at (0,0);
        \fill (0,0) circle (1pt);
        \node[anchor=north east] at (0,0) {\includegraphics[width=8mm]{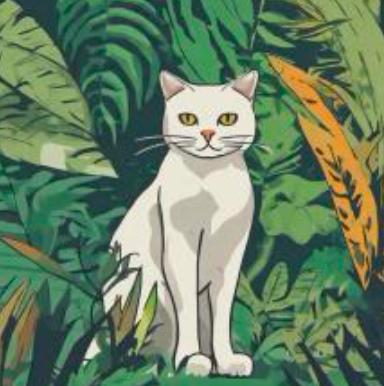}};
      },
      mark=at position 0.4 with {
        \coordinate (P2) at (0,0);
        \fill (0,0) circle (1pt);
        \node[anchor=north east] at (0,0) {\includegraphics[width=8mm]{imgs/cat1.jpg}};
      },
      mark=at position 0.6 with {
        \coordinate (P3) at (0,0);
        \fill (0,0) circle (1pt);
        \node[anchor=north east] at (0,0) {\includegraphics[width=8mm]{imgs/cat1.jpg}};
      },
      mark=at position 0.8 with {
        \coordinate (P4) at (0,0);
        \fill (0,0) circle (1pt);
        \node[anchor=north east] at (0,0) {\includegraphics[width=8mm]{imgs/cat1.jpg}};
      }
    },
    postaction={decorate}
  },
  decoB/.style={
    decoration={
        markings,
          mark=at position 0.2 with {
            \coordinate (P5) at (0,0);
            \fill (0,0) circle (1pt);
            \node[anchor=south west] at (0,0) {\includegraphics[width=8mm]{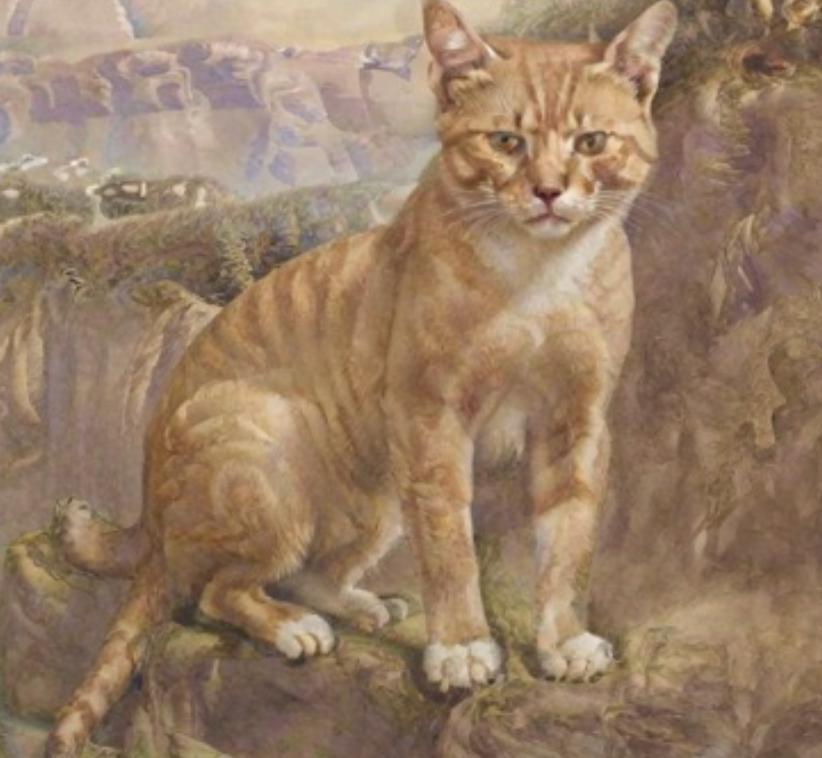}};
          },
          mark=at position 0.4 with {
            \coordinate (P6) at (0,0);
            \fill (0,0) circle (1pt);
            \node[anchor=south west] at (0,0) {\includegraphics[width=8mm]{imgs/cat2.jpg}};
          },
          mark=at position 0.6 with {
            \coordinate (P7) at (0,0);
            \fill (0,0) circle (1pt);
            \node[anchor=south west] at (0,0) {\includegraphics[width=8mm]{imgs/cat2.jpg}};
          },
          mark=at position 0.8 with {
            \coordinate (P8) at (0,0);
            \fill (0,0) circle (1pt);
            \node[anchor=south west] at (0,0) {\includegraphics[width=8mm]{imgs/cat2.jpg}};
          }
        },
        postaction={decorate}
    }
}

\usepackage{tikz-cd}

%% file: sec/0_abstract.tex
\begin{abstract}
In a rapidly growing field of model training there is a constant practical interest in parameter-efficient fine-tuning and various techniques that use a small amount of training data to adapt the model to a narrow task. However, there is an open question: how to combine several adapters tuned for different tasks into one which is able to yield adequate results on both tasks? Specifically, merging subject and style adapters for generative models remains unresolved. In this paper we seek to show that in the case of orthogonal fine-tuning (OFT), we can use structured orthogonal parametrization and its geometric properties to get the formulas for training-free adapter merging. In particular, we derive the structure of the manifold formed by the recently proposed Group-and-Shuffle ($\mathcal{GS}$) orthogonal matrices, and obtain efficient formulas for the geodesics approximation between two points. Additionally, we propose a \emph{spectra restoration} transform that restores spectral properties of the merged adapter for higher-quality fusion. We conduct experiments in subject-driven generation tasks showing that our technique to merge two $\mathcal{GS}$ orthogonal matrices is capable of uniting concept and style features of different adapters. To the best of our knowledge, this is the first training-free method for merging multiplicative orthogonal adapters. Code is available via the link: \url{https://github.com/ControlGenAI/OrthoFuse}.
\end{abstract}

%% file: sec/1_intro.tex
\section{Introduction}
\label{sec:intro}
The impressive generative abilities of diffusion models~\citep{stablediffusion,flux2024} have fueled increasing interest in subject-driven generation~\citep{DB,TI,CD} and stylization tasks~\citep{styledrop}. Typically, such tasks are addressed by fine-tuning pre-trained models using a small dataset corresponding to a specific concept, which can be a particular object, individual, or artistic style. 

While current approaches made it possible to handle subject- and style-oriented generation independently, a critical challenge remains unresolved: generating images combining a user-defined subject with a user-defined artistic style. However, addressing this problem would significantly expand the capabilities of generative models and decrease training costs, offering a new level of control and creative freedom to users.

Despite the promising advancements in merging LoRA~\citep{lora} adapters, current research has yet to explore methods for combining multiplicative orthogonal adapters, recently introduced in \citep{qiu2023controlling}. At the same time, orthogonal adapters have been shown to deliver more stable training dynamics and reduce the risk of overfitting, making them a particularly appealing option for personalization and stylization tasks. 
Another notable benefit of orthogonal fine-tuning is that, compared to LoRA, multiplicative orthogonal adapters are able to preserve the spectral and Frobenius norm of the layer by design, which is hard to achieve in the case of LoRA. This unique property enables the seamless fusion of orthogonal adapters without any concern regarding differences in their magnitudes. However, the absence of studies focusing on the fusion of orthogonal adapters represents a notable gap in this domain, presenting a promising research direction for advancing more robust and effective subject-style synthesis in generative models.

In this work, we introduce a novel training-free approach for merging orthogonal adapters built upon the foundations of \lpr-orthogonal parametrization proposed in~\citep{gorbunov2024group}. By studying the properties of this matrix class, we demonstrate that they form a Riemannian manifold. This structure allows us to approximate geodesics between two orthogonal adapters and explore other meaningful curves on the manifold, ultimately facilitating the construction of an optimally merged adapter.

Additionally, we examine the spectral properties of the resulting merged adapters and establish a strong correlation between the spectral distribution and the merging quality. A naive geodesic approach tends to produce a more ``compressed-to-one'' spectral distribution, which diminishes the expressiveness of the outputs by limiting both concept representation and style preservation. 
To overcome this issue, we propose applying the Cayley transform to the resulting curve, thereby restoring the spectral distribution and achieving a superior, optimized merge and generation quality (Figure~\ref{fig:visual_abs}).
Overall our key contributions are as follows:
\begin{itemize}
    \item We investigate the set of \lpr-orthogonal matrices and show this set forms a Riemannian manifold.
    \item Using these theoretical insights, we propose a method to approximate geodesics between \lpr-orthogonal matrices.
    \item Through extensive experiments, we show that our training-free approach effectively fuses parameter-efficient orthogonal adapters, successfully combining both style and concept patterns. Our method outperforms existing state-of-the-art LoRA-based approaches as well as straightforward joint orthogonal fine-tuning.
    \item To the best of our knowledge, we are the first to propose a training-free merging method specifically designed for orthogonal multiplicative adapters.
\end{itemize}

%% file: sec/2_related_work.tex
\section{Related Work}
\label{sec:related_work}

\paragraph{LoRA vs Orthogonal Fine-tuning.}
LoRA-based methods~\cite{lora,dora,tlora} have become the standard approach for parameter-efficient fine-tuning. However, prior work~\cite{qiu2023controlling} shows that additive low-rank updates may distort neuron relationships, which are important for preserving generative semantics in diffusion models. To address this, orthogonal fine-tuning methods introduce multiplicative updates that preserve weight geometry via orthogonal parametrizations. 

While the original block-diagonal construction~\citep{qiu2023controlling} is efficient, it limits interactions across parameter groups. To improve expressiveness,~\citep{gorbunov2024group} propose combining multiple orthogonal blocks with permutations, enabling richer transformations while retaining computational efficiency.

\paragraph{Adapter Merging.}
Merging independently trained adapters is commonly studied in the context of LoRA. Early approaches rely on simple weighted averaging~\citep{Ryu_lora}, while more recent methods introduce either learnable or structured composition strategies. 

Training-based approaches such as MoLe~\citep{wu2024mixture} and ZipLoRA~\citep{ziplora} learn how to combine multiple LoRA adapters via additional optimization, using layer-wise gating (MoLe) or fine-grained column-wise mixing coefficients (ZipLoRA). B-LoRA~\citep{blora} leverages architectural modularity by selectively training specific components (e.g., attention layers) to better separate and combine content and style. In contrast, training-free methods such as K-LoRA~\citep{klora} select adapters based on weight statistics, avoiding retraining but remaining sensitive to scale inconsistencies between independently trained LoRAs.

\paragraph{Test-time and Representation-level Methods.}
Training-free personalization can also be achieved by modifying the generation process at inference time. For instance, RB-Modulation~\citep{rout2024rb} steers reverse diffusion dynamics using reference-based objectives, without relying on adapter parametrizations or explicit adapter merging. In contrast, our work focuses on combining pre-trained adapters directly in parameter space via closed-form fusion.

Subject--style separation is often approached via feature-level control mechanisms (e.g., StyleAligned~\citep{hertz2024style}, StyleDrop~\citep{sohn2023styledrop}), which guide generation by modulating internal representations. In contrast, OrthoFuse exploits the geometric structure of orthogonal weight manifolds, enabling composition directly in parameter space and avoiding scale inconsistencies of LoRA-based methods.

Our method is also related to optimization on matrix manifolds~\cite{GHL, AbsMahSep2008}; however, applying Riemannian geometry to merge orthogonal adapters for diffusion models remains largely unexplored.

%% file: sec/3_preliminaries.tex
\section{Preliminaries}
\label{sec:preliminaries}
\subsection{Diffusion Model Fine-tuning}
Fine-tuning based personalized image generation is a process of adapting the model's weights to generate user-defined concept. To link a new concept to a unique text token and class name, for example ``sks dog'', the model $\varepsilon_{\theta}$ undergoes fine-tuning on a limited dataset of concept images $\mathrm{C} = \{x\}_{i=1}^N$, optimizing the following objective:
\begin{equation}\label{eq:finetuning} \min_\theta \mathbb{E}_{p, t, z=\mathcal{E}(x), x \in \mathrm{C}, \varepsilon}\left[\left\|\varepsilon - \varepsilon_\theta(t, z_t, p)\right\|_2^2\right], \end{equation} 
where $\mathcal{E}$ denotes the encoder that maps an image $x$ to its latent representation $z=\mathcal{E}(x)$.

In the case of orthogonal fine-tuning~\cite{gorbunov2024group}, we exclusively optimize the multiplicative adapter $A$, which is \lpr-orthogonal matrix, while keeping the original model weights $W$ fixed. As a result, the updated weights are defined as: $W^* = A W$.

\label{sec:fine-tuning}

\subsection{\texorpdfstring{\(\mathcal{G}\mathcal{S}\)}{GS}-orthogonal matrices}
\label{sec:gs}

Let us start with a definition of the \lpr-orthogonal matrix, which we use as the main building block of our orthogonal adapter in our work. We introduce a simplified definition that is used in our experiments, for the more general case see~\citep{gorbunov2024group}. 
Recall also that a square matrix $A$ is called orthogonal if it satisfies $A^\top A = I$.

\begin{definition}[\citep{gorbunov2024group}]\label{def:gs_orthogonal}
An $n \times n$ orthogonal matrix $A$ is called \lpr$(P_L, P, P_R)$-orthogonal matrix with block size $b \times b$ if it can be represented in the following form:
\begin{equation}\label{eq:gs_definition}
    A = P_L L P R P_R,
\end{equation}
where $L = \diag(L_1, L_2, \dots, L_{\frac{n}{b}})$ and $R = \diag(R_1, R_2, \dots, R_{\frac{n}{b}})$ are block-diagonal matrices with $L_i, R_i \in \mathbb{R}^{b\times b}$, and $P_L, P, P_R$ are permutation matrices.
\end{definition}

\cref{th:gs_ort} shows that without the loss of generality we may assume that the diagonal blocks of \( L \) and \( R \) are set to be orthogonal.

\begin{theorem}[\citep{gorbunov2024group}]\label{th:gs_ort}
    Let $A$ be any orthogonal matrix from \lpr$(P_L, P, P_R)$.
    Then, $A$ admits $P_L(LPR)P_R$ representation with the matrices $L, R$ consisting of orthogonal blocks. 
\end{theorem}

A certain choice of the permutation matrices $P_L, P, P_R$ depends on the application. 
For example, we set \( P_R = I \) and \( P_L = P^\top \), which ensures that the identity matrix \( P_L I P I P_R = I \) belongs to our class --- a property vital for initialization.
In this paper $P$ is chosen to be the so-called \emph{perfect shuffle}~\citep{GL,dao2022monarch}, which maximizes the number of non-zero elements in the matrix~\citep{gorbunov2024group}. 
Following \citep{dao2022monarch}, we denote perfect shuffle permutations as $P_{(b, n)}$. Applying permutation $P_{(b, n)}$ can be interpreted as the following procedure: firstly, it reshapes a length-$n$ vector into a $b \times \frac{n}{b}$ matrix in row-major order. Then, this matrix is transposed and flattened back into a vector (again in row-major order).
For simplicity of notation, we replace \lpr$(P_{(b, n)}^\top, P_{(b, n)}, I)$ with \lpr~where it is unambiguous. 
We also note that this representation resembles Monarch matrices~\citep{dao2022monarch}, which, however, did not consider orthogonality constraints.

In this paper, we maintain block orthogonality during fine-tuning via the Cayley transform: for every block $B$ of block-diagonal $L$ or $R$ the operation
\[
    B = (I-K)^{-1} (I+K), \quad K^\top = - K.
\]
yields an orthogonal matrix from the special orthogonal group $\mathrm{SO}(N)$ of orthogonal matrices with the determinant equal to $1$.

\subsection{Riemannian geometry}

An \(n\)-dimensional topological manifold \(\mathcal{M}\) is a space that is locally homeomorphic to \(\mathbb{R}^n\). It means that for every point \(p \in \mathcal{M}\), there exists an open neighborhood \(U \ni p\) and a homeomorphism \(\varphi: U \to V\), where \(V \subset \mathbb{R}^n\) is open; the pair \((U, \varphi)\) is called a chart.

A smooth manifold is a topological manifold equipped with an atlas (a collection of charts covering \(\mathcal{M}\)) such that all transition maps \(\psi \circ \varphi^{-1}\) between overlapping charts are smooth. This additional structure allows for defining smooth functions, curves, and maps on the manifold.

A Riemannian manifold \((\mathcal{M}, g)\) is a smooth manifold endowed with a Riemannian metric \(g\), which is a smoothly varying family of inner products \(g_p\) on each tangent space (denoted \(T_p \mathcal{M}\)). This metric makes it possible to measure lengths of tangent vectors, angles between tangent vectors, and lengths of curves. By integrating the length of curves, one obtains an intrinsic distance function \(d_g: \mathcal{M} \times \mathcal{M} \to \mathbb{R}\), turning \(\mathcal{M}\) into a metric space.

On a Riemannian manifold, geodesics generalize the notion of straight lines from Euclidean geometry. A geodesic is a curve \(\gamma: (-\varepsilon, \varepsilon) \to \mathcal{M}\) that is locally length-minimizing: for any two sufficiently close points on \(\gamma\), the curve realizes the shortest path between them, as measured by \(d_g\).

A Lie group is a smooth manifold that is also a group, in which the group operations of multiplication and inversion are smooth maps.
For more details see \cite{Lee, AbsMahSep2008, Vinberg_eng}.

It is also a well-known fact (see, e.g., \cite{Lee}) that some matrix classes, like orthogonal group \( \mathrm{O}(N) \), special orthogonal group 
\( \mathrm{SO}(N) = \{ A \in \mathrm{O}(N) \mid \det(A) = 1 \} \subseteq \mathbb{R}^{N\times N}\) and the set of fixed-rank matrices are smooth manifolds. This allows for leveraging their geometric properties in optimization tasks. For instance, one can interpret neural network weights as points on a manifold and apply Riemannian optimization methods~\cite{AbsMahSep2008} for efficient training.

%% file: sec/4_method.tex
\section{Method}
\label{sec:method}
\begin{figure*}[!h]
  \centering    \includegraphics[width=0.75\linewidth]{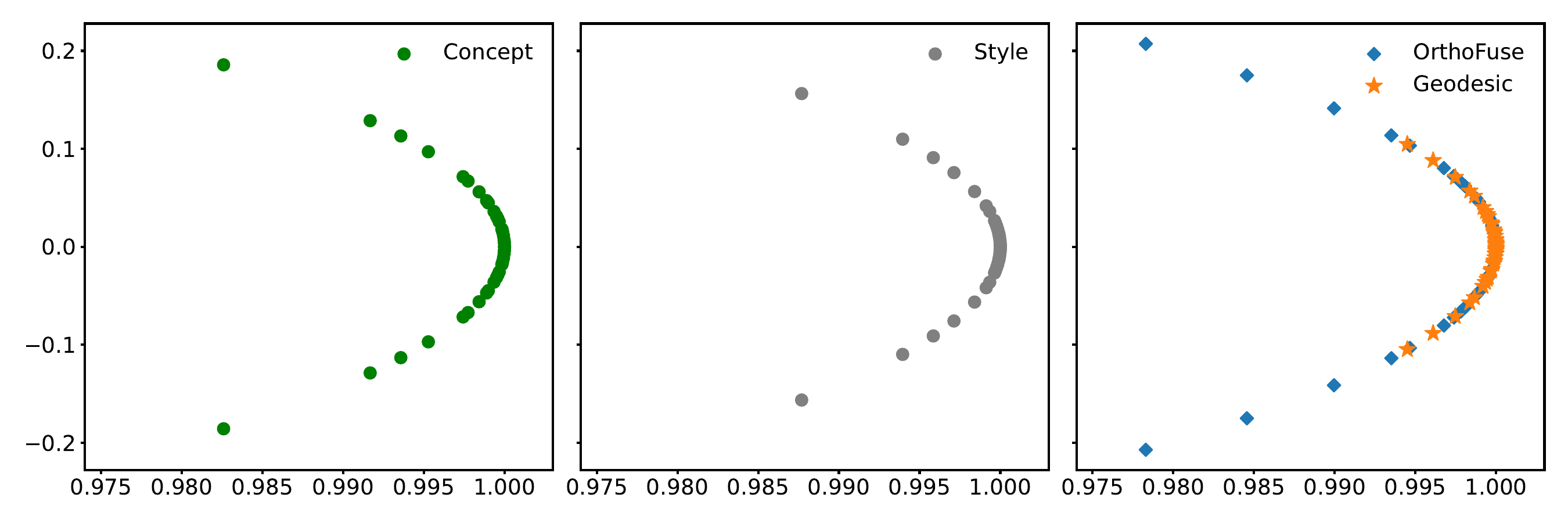}
    \caption{Distribution of the eigenvalues of orthogonal fine-tuning adapters before and after merging. (Left): eigenvalues of $L_C^1$. (Center): eigenvalues of $L_S^1$. (Right): eigenvalues of $B(0.5)$ (orange) and of $B_{\text{OrthoFuse}}(0.5)$ (blue). Eigenvalues are calculated for the query \lpr-orthogonal adapter of StableDiffusion-XL model.}
\label{fig:spectrum_compression}
\end{figure*}

In this work, we aim to merge two \lpr-orthogonal adapters so that we obtain an adapter from the same structured class.
To achieve this without additional training, we establish the geometry of  \lpr-orthogonal matrices class.
The following theorem allows us to treat $\mathcal{GS}$-orthogonal matrices as manifold elements, which is our key theoretical result.

\begin{theorem}\label{th:gs_smooth_manifold}
     The set of \lpr\((P_L, P, P_R)\)-orthogonal matrices forms a smooth manifold.
\end{theorem}

\begin{proof}
   See Appendix \ref{sec:smoothness_proof}.
\end{proof}

\cref{th:gs_smooth_manifold} provides a way to connect manifold elements via interpretable curves. 
We show that, with the right choice of the curve, we can obtain a gradual transfer between adapters, allowing us to mix concept and style in the desired proportion. This curve will serve as the reasonable approximation of the local minimizing geodesics between two points.
Now, assume that we have two \lpr-orthogonal matrices:
\begin{equation} \label{eq:adapters_general_form}
A_C = P^\top L_C P R_C, \quad A_S = P^\top L_S P R_S,
\end{equation}
where $A_C$ and $A_S$ are weight update matrices trained on a certain concept and style respectively.

One may assume that the local minimizing geodesics between two $\mathcal{GS}$-orthogonal matrices is just a block-wise interpolation between the corresponding diagonal blocks of $(L_C, L_S)$ and $(R_C, R_S)$. However, the $\mathcal{GS}$-orthogonal manifold exhibits a more complicated structure, and the local minimizing geodesics between two manifold points is resource-intensive to calculate.
Fortunately, orthogonal fine-tuning yields matrices whose diagonal blocks lie close to the identity matrix. 
This empirical observation, which was also reported in \citep{qiu2023controlling}, allows us to conclude that block-wise geodesics interpolation precisely approximates the exact local minimizing geodesics on the \lpr-orthogonal manifold. See Appendix~\ref{sec:merging_proof} for more details.

Now, let us see the procedure of connecting blocks in more detail. On $\mathrm{SO}(n)$, a geodesic starting at $B_C$ is $B(t) = B_C \exp(t\Omega)$ with $\Omega$ skew-symmetric (see \cite{boumal2023optimization, AbsMahSep2008, EAS98, Lee_math_stack_exchange}). To make the geodesic reach $B_S$ at $t=1$ for an arbitrary pair of corresponding blocks $B_C, B_S \in \mathrm{SO}(n)$, we set $\Omega = -\log(B_S^{\top}B_C)$. Thus we obtain the well‑known formula
\begin{equation}\label{eq:lmg}
    B(t) = B_C\exp(-t \cdot \log(B_S^{\top}B_C)),
\end{equation}
where $t \in [0, 1]$ and $\exp$ and $\log$ denote matrix exponent and matrix logarithm functions respectively.

In practice, $B(t)$ can be computed using the eigendecomposition $B_S^{\top}B_C = U\Lambda U^*$ (since orthogonal matrix is always diagonalizable):
\begin{equation}
\label{eq:geo}
       B(t) = B_C U \exp(-t \cdot \log(\Lambda)) U^*, 
\end{equation}
which reduces to matrix functions that are separately applied to eigenvalues, and matrix multiplications that are highly efficient on GPUs.

The fact that we apply our fusing operation for $A_C$ and $A_S$ block-wise plays a vital role in the resulting algorithm efficiency due to cubical time scaling of eigendecomposition.

We empirically observe that merging orthogonal adapters draws the eigenvalues of the resulting matrix closer to $1$ compared to the original components (see Figure~\ref{fig:spectrum_compression}). Since these eigenvalues control the orthogonal adapter’s rotation power, their convergence toward unity weakens the intended stylistic modifications, making it closer to an identity layer transform. To address this attenuation of style and to proactively enhance the adapter's effect, we propose a \emph{spectra restoration} procedure.

A candidate for \emph{spectra restoration} is a rotation of the eigenvalues on the complex unit circle. This operation corresponds to multiplying the complex phase of each eigenvalue by a scalar factor. For eigenvalues lying on a unit sphere near the point $1$, the phase can be extracted by taking the logarithm, yielding a value in a subsegment  $i \cdot (-\alpha, \alpha) \subset i \cdot (-\pi, \pi)$ providing that initial orthogonal matrix is close to $I$. 

Formally, we propose the following approach:
\begin{equation} \label{eq:rotation}
    B_{\text{rotated}}(t) = \exp(\eta(t) \log(B(t))), 
\end{equation}
where $\eta(t)$ is smooth phase multiplier satisfying
\begin{equation} \label{eq:etafunc}
    \eta(0) = \eta(1) = 1, \quad\text{and}\quad \eta(1/2) = \eta_0,
\end{equation}
with $\eta_0$ being a hyperparameter. The condition $\eta(0) = \eta(1) = 1$ is to ensure that we restore initial adapters on the boarder. Based on our ablation studies (see Appendix~\ref{sec:app_abl}), we empirically found that a suitable solution consists of $\eta_0 = 2$ and a second-order polynomial $\eta(t)$ satisfying property~\eqref{eq:etafunc}:
\begin{equation}\label{eq:final_etafunc}
    \eta(t) = 1 + 4 t(1-t).
\end{equation}

\begin{figure*}[t!]
  \centering
    \includegraphics[width=0.95\linewidth]{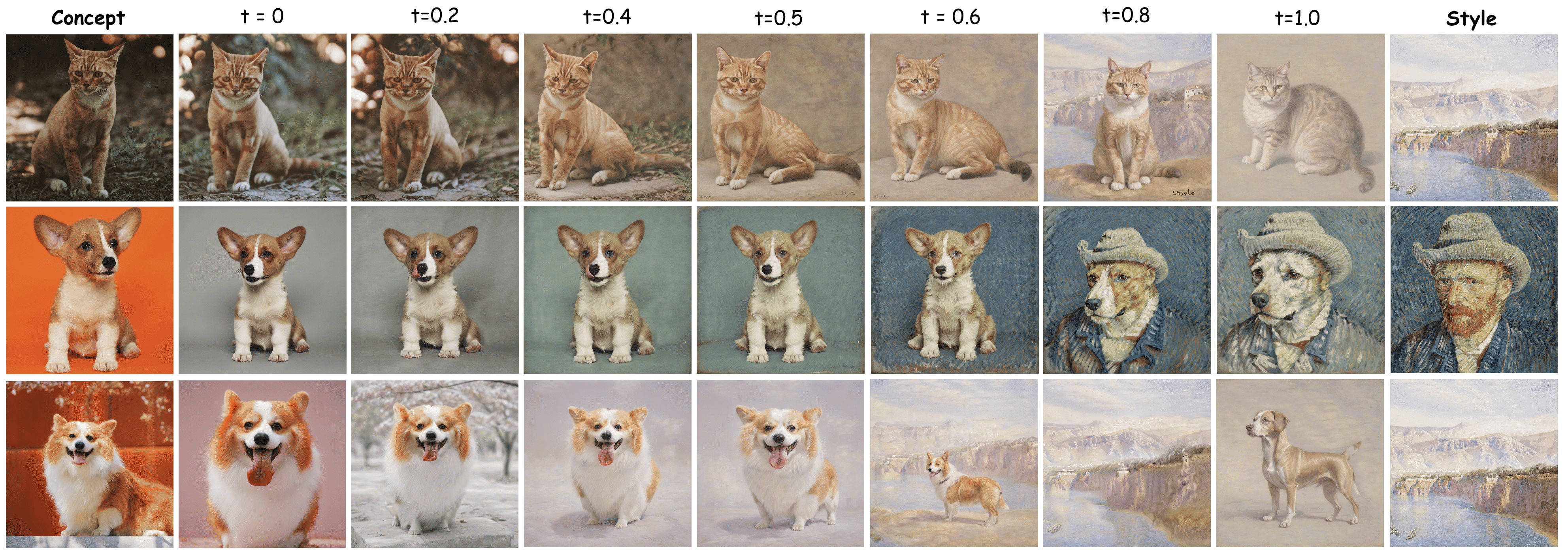}
    \caption{Ablation of the fusion parameter $t$ from ~\eqref{eq:etafunc}. When $t = 0$, the merged adapter reduces to a pure concept adapter, preserving identity with no stylization. When $t = 1$, the merged 
weights correspond to a pure style adapter. 
Intermediate values produce a continuous fusion curve between concept preservation and 
style strength, with $t = 0.6$ yielding the most balanced trade-off.}
\label{fig:t_ablation}
\end{figure*}

The only problem with such an approach is that it is computationally demanding and requires matrices to be diagonalized for every $t$ by contrast to~\eqref{eq:lmg}. 
To improve computational efficiency, while staying close to~\eqref{eq:rotation}, we do two approximation steps. First of all, we use the following straightforward propositions.
\begin{prop}\label{prop:log_orth} For a matrix $B(t) \in \mathrm{SO}(N)$  the following equality holds:
\[
\log(B(t)) = \frac{B(t) - B(t)^\top}{2} + \mathcal{O}\left(\|B(t) - I\|_2^3\right),
\]
as $B(t)$ tends to $I$.
\end{prop}
\begin{proof}
See Appendix \ref{sec:proof_prop1}.
\end{proof}

\begin{prop}[] \label{prop:exp_pade}
\[
\exp(tK) = \left(I - \frac{t}{2}K\right)^{-1} \left(I + \frac{t}{2}K\right) + \mathcal{O}(t^3),
\]    
as $t\to 0$. 
\end{prop}

\begin{proof}
    See~\citep[Chapter 11.3.1]{GL} (Padé Approximation Method for Exponent for \(p=q=1\)).
\end{proof}

As a result, we arrive at the following proposition yielding a hardware-efficient way to apply \emph{spectra restoration}.

\begin{prop}\label{prop:orthofuse}
    Let $\eta(t) = 1 + 4t(1-t)$ and $\mathrm{det}(B(t) + I) \not= 0$. Then for 
    \begin{equation}\label{eq:final_merge}
\begin{split}
    B_{\text{OrthoFuse}}(t) & = \left(I - \frac{\eta(t)}{4}(B(t) - B(t)^{\top})\right)^{-1} \cdot \\ \cdot & \left(I + \frac{\eta(t)}{4}(B(t) - B(t)^{\top})\right),
\end{split}
\end{equation}
it holds that
\[
    B_{\text{OrthoFuse}}(t) = B_{\text{rotated}}(t) + \mathcal{O}\left(\|B(t) - I\|_2^3\right),
\]
as $B(t)\to I$ and where $B_{\text{rotated}}$ is defined in \eqref{eq:rotation}
\end{prop}
\begin{proof}
    See Appendix \ref{sec:proof_prop3}.
\end{proof}

To sum up, in our merging procedure we add a specific transformation for each diagonal block $B(t)$ merged via local minimizing geodesics from~\eqref{eq:lmg}. 
As a result, the spectrum of the resulting blocks is modified using \eqref{eq:final_merge} with $\eta(t)$ defined in \eqref{eq:final_etafunc}.

\subsection{Practical Implementation}

Our method operates on top of any diffusion model whose layers can be fine-tuned with orthogonal adapters.

\paragraph{Source of adapters.}

Both style and concept adapters are trained using the orthogonal parametrization described in 
Section~\ref{sec:gs}: using the same notation as in \eqref{eq:gs_definition}, we express concept $A_C$ and style $A_S$ adapters as follows:
\begin{equation}
\begin{split}
& A_C = P_{(b, n)}^\top L_C P_{(b, n)} R_C, \\ 
& A_S = P_{(b, n)}^\top L_SP_{(b, n)} \, R_S,
\end{split}
\end{equation}
where 
\begin{equation}
\begin{split}
L_C = \mathrm{diag}(L_C^{(1)}, \dots, L_C^{(\frac{n}{b})}),\   & R_C = \mathrm{diag}(R_C^{(1)}, \dots, R_C^{(\frac{n}{b})}), \\
L_S = \mathrm{diag}(L_S^{(1)}, \dots, L_S^{(\frac{n}{b})}),\  & R_S = \mathrm{diag}(R_S^{(1)}, \dots, R_S^{(\frac{n}{b})}).
\end{split}
\end{equation}

\paragraph{Goal.} Given $A_C$ and $A_S$, the task is to construct a fused adapter $A(t)$ controlled by the fusion parameter $t \in [0,1]$, such that $A(t)$ belongs to \lpr-orthogonal matrices class and contains a mixture of features extracted from $A_C$ and $A_S$ in particular proportion.
\paragraph{Algorithm.} The merge operation is performed 
independently for each pair of blocks that are either $(L_C^{(i)}, L_C^{(i)})$ or $(R_C^{(i)}, R_S^{(i)})$.  
The merging procedure consists of two steps:
\begin{enumerate}
    \item \textbf{Block-wise geodesic interpolation.}
Since the adapters consist of independent orthogonal blocks, the merging operation 
is performed block by block. For each pair of blocks $(B_C^{(i)}, B_S^{(i)})$ we 
compute their fused version using the block-wise geodesic interpolation defined in~\eqref{eq:geo}:
\begin{equation}
\widetilde{B}^{(i)}(t) = B_C^{(i)} U \exp(-t \cdot \log(\Lambda)) U^\top, 
\end{equation}
where $B_S^{\top}B_C = U\Lambda U^\top$.
This produces an intermediate block corresponding to the fusion level $t$ along 
the geodesic between the concept and style transformations.

\item \textbf{Eigenvalue rotation.}
After going along the geodesic, we apply the eigenvalue rotation
operation used during fine-tuning, described in \eqref{eq:final_merge}:
\begin{equation}
\begin{split}
B^{(i)}(t) & = \left(I - \frac{\eta(t)}{4}(\widetilde{B}^{(i)}(t) - \widetilde{B}^{(i)}(t)^{\top})\right)^{-1} \cdot \\ \cdot & \left(I + \frac{\eta(t)}{4}(\widetilde{B}^{(i)}(t) - \widetilde{B}^{(i)}(t)^{\top})\right).
\end{split}
\end{equation}
As a result, $B^{(i)}(t)$ serves as a final block of the merged style and concept adapter.

\end{enumerate}
Additionally, we implement an accelerated version of this algorithm that merges two \lpr-orthogonal adapters in under one second. Details are provided in Appendix \ref{sec:app_code}, where we also include pseudocode for both the OrthoFuse algorithm and its accelerated variant.
\begin{table*}[t]
  \caption{Quantitative comparison of OrthoFuse and baseline methods. 
Style Sim measures style fidelity using CLIP similarity with the reference style image. 
CLIP and DINO concept metrics evaluate semantic consistency with the original concept. 
The geometric mean combines style and concept metrics to summarize overall trade-off between style fidelity and concept preservation.}
  \label{tab:quant_anal}
  \centering
  \begin{tabular}{@{}lcccccc@{}}
    \toprule
    Method & Style Sim & CLIP & DINO &Geo. Mean (Style, DINO) & Geo. Mean (Style, CLIP) & Merging time\\
    \midrule
    {$\quad$\textit{Training-based}} \\
    \midrule
    Joint training &  0.48 & \textbf{0.79} & \textbf{0.67} & \textbf{0.57} & 0.62 & 1.5 hours \\
    ZipLoRA $r=8$ & 0.49 & 0.74 & 0.55 & 0.52 & 0.6  & \multirow{2}{*}{4 minutes}\\
    ZipLoRA $r=64$ & 0.49 & 0.76 & 0.64 & 0.56 & 0.61 & \\
    \midrule
    {$\quad$\textit{Training-free}} \\
    \midrule
     K-LoRA $r=8$ &  0.46 & 0.76& 0.55 & 0.5 & 0.59 & \multirow{2}{*}{$<1$ sec}\\
    K-LoRA $r=64$  & 0.49 & 0.76 & 0.56 & 0.52 & 0.61 & \\
    \textbf{OrthoFuse} & \textbf{0.61} & 0.68 & 0.51 & 0.56 & \textbf{0.64} & $<1$ sec\\
    
    \bottomrule
  \end{tabular}

\end{table*}

%% file: sec/5_experiments.tex
\section{Experiments}

\begin{figure*}[t!]
  \centering
    \resizebox{0.80\textwidth}{!}{\includegraphics[width=0.95\linewidth]{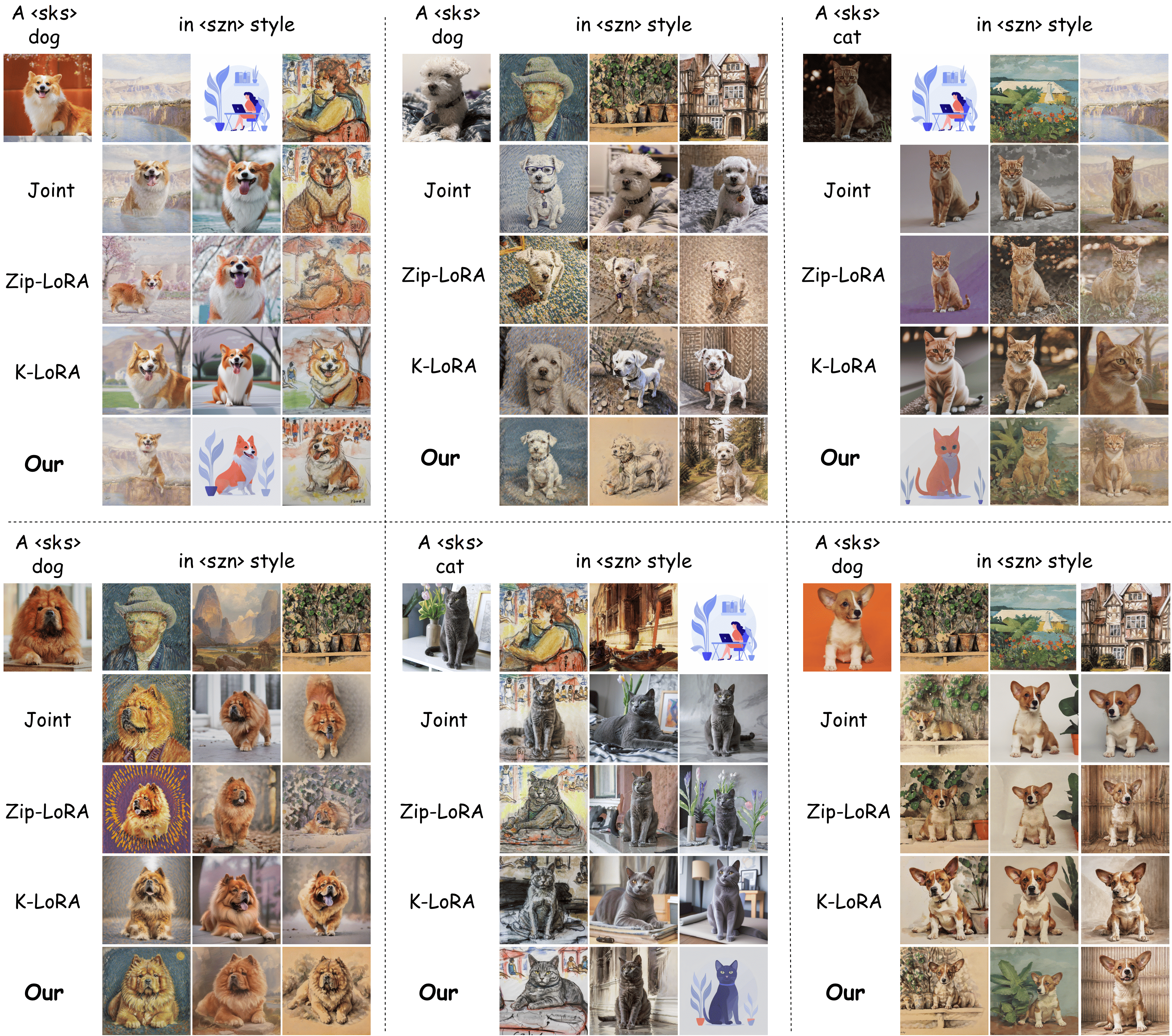}}
    \caption{\textbf{Qualitative comparisons.} We present images generated by OrthoFuse alongside those created using the compared baselines. OrthoFuse strikes an ideal balance between concept and style, preserving both the concept and style.}
\label{fig:qualitative}
\end{figure*}

\subsection{Datasets}

To evaluate the effectiveness of OrthoFuse, we used a diverse set of styles and object concepts.
The style adapters were trained on styles from StyleDrop~\cite{styledrop} and from the artistic style collection used in K-LoRA~\cite{klora}, covering both classical and contemporary artworks.
The concept adapters were trained on object concepts sampled from the DreamBooth dataset. This setup allows us to assess OrthoFuse on a wide variety of visual domains and to validate its ability to preserve both stylistic and semantic attributes during fusion.

\subsection{Experimental Details}

All experiments were conducted using the SDXL~\citep{sdxl} model as the base model. Additional results for FLUX~\citep{flux} are reported in Appendix \ref{sec:flux_results}.

Each adapter was trained separately. Concept adapters were fine-tuned using $4$-$5$ images per subject, while style adapters were trained on a single style reference image. All adapters were trained with the number of blocks set to $32$.

Our ablation study in Figure~\ref{fig:t_ablation} demonstrates that the value $t = 0.6$ achieves 
the most robust fusion behavior across a wide range of style-concept combinations, 
providing stable style transfer while maintaining strong identity consistency. Unless stated otherwise, we report results obtained with $t=0.6$. Additional insights regarding the impact of varying $t$ can be found in the Appendix~\ref{sec:app_abl}, where we provide a more detailed examination of the effects of this parameter.

Additionally, a comparison of generations obtained after block-wise geodesic interpolation alone versus after the subsequent eigenvalue rotation step is provided in Appendix~\ref{sec:blockwise_vs_rotation}, illustrating the effect of the rotation on stylistic fidelity and concept preservation.

\subsection{Evaluation Metrics}

We evaluate OrthoFuse using both semantic and stylistic similarity metrics. To assess concept preservation, we compute the CLIP similarity and DINO similarity between original and generated images. To evaluate style fidelity, we calculate the CLIP similarity between the generated images and the reference style image. This metric quantifies how well the artistic characteristics of the target style are transferred during the fusion process.

\subsection{Quantitative Comparisons}

We compare OrthoFuse with three representative baselines: K-LoRA \citep{klora}, ZipLoRA \citep{ziplora} and Joint Orthogonal Adapters Training (Joint). For a fair comparison, we evaluated these methods using two different LoRA ranks: (1) Rank $r=8$, as in the K-LoRA~\citep{klora} original paper; (2) Rank $r=64$, chosen as the rank for original ZipLoRA \citep{ziplora} method and to match the number of parameters used in our orthogonal adapters. Note that a LoRA of rank $32$ roughly corresponds to an orthogonal adapter with the number of blocks set to $64$ in terms of parameter count.

For the quantitative analysis, we used $6$ concepts from the DreamBooth dataset and $12$ styles from StyleDrop and K-LoRA. This resulted in 72 concept-style combinations; for each, we generated 10 images per pair for evaluation.

Table~\ref{tab:quant_anal} reports the quantitative results averaged over all concept-style pairs. 
OrthoFuse achieves the highest CLIP style similarity, demonstrating the strongest style transfer among all compared methods. While its concept preservation metrics (CLIP and DINO similarity) are slightly lower than those of the best-performing baseline, this behavior is expected because strong stylization always moves the generation away from the original concept images. Additionally, the geometric mean of style and concept similarity is highest for OrthoFuse, indicating the best overall balance between style fidelity and concept retention.

The \textbf{Joint} baseline, which trains an orthogonal adapter for both concept and style at the same time, achieves the highest concept-preservation scores (CLIP and DINO); however, we found out that it can favor the concept and ignore the style in several cases, as reflected in its low style similarity. Moreover, this method is training-based and requires additional fine-tuning for each concept-style pair, whereas \textbf{OrthoFuse} is entirely training-free.

Compared to \textbf{K-LoRA} and \textbf{ZipLoRA}, OrthoFuse provides significantly stronger style transfer while maintaining competitive concept consistency, offering a more stable and reliable fusion across diverse concept–style combinations.

\subsection{Qualitative Results}

Figure~\ref{fig:qualitative} presents qualitative comparisons between OrthoFuse and the baseline fusion methods. 
OrthoFuse consistently finds a balance between style transfer and concept preservation, maintaining the semantic identity of the concept while accurately reflecting the target artistic style.

In contrast, baseline methods exhibit clear limitations. 
ZipLoRA and K-LoRA often produce artifacts or fail to transfer stylistic features faithfully, particularly for challenging styles, and their results depend heavily on the specific concept–style pair, making them unstable across different combinations. 
Joint preserves the concept very well but struggles to apply the target style effectively, resulting in weaker stylization.

Overall, OrthoFuse generates visually coherent compositions with consistent textures, lighting, and stylistic patterns, even for difficult style-concept combinations. 
These qualitative observations align with the quantitative findings, confirming that OrthoFuse achieves the most balanced integration of concept and style, producing images that are both semantically faithful and aesthetically rich. Additional qualitative results are provided in the Appendix \ref{sec:qual_app}.

\subsection{User Study}

\begin{table}

  \caption{
\textbf{User study}. Comparison of our method with K-LoRA and ZipLoRA on all images used in the quantitative evaluation. 65 participants, 1460 pairwise comparisons.
}
  \label{tab:user_study}
  \centering
  \begin{tabular}{lcc}
    \toprule
    \textbf{Question} & \textbf{Our vs K-L.} & \textbf{Our vs Z.} \\
    \midrule
    Concept Preserv. (Q1) &  48\% \text{vs} 52\% & 54\% \text{vs} 46\%\\
    Style Transfer (Q2) & 77\% \text{vs} 23\% & 83\% \text{vs} 17\%\\
    \midrule[0.1pt]
    Overall Preference (Q3) & \textbf{67\%} \text{vs} 33\% & \textbf{76\%} \text{vs} 24\% \\
    \bottomrule
  \end{tabular}

\end{table}

To complement automatic metrics and account for their known limitations in evaluating concept–style trade-offs, we conducted a user study comparing our method with K-LoRA and ZipLoRA. Participants were asked three questions: Q1 evaluated \textit{Concept Preservation}, Q2 assessed \textit{Style Transfer}, and Q3 measured \textit{Overall Preference}. Full protocol details and question wording are provided in the Appendix~\ref{sec:app_user_study}.

We collected responses from 65 participants, resulting in 1,460 pairwise comparisons across all images used in the quantitative evaluation, with half of the comparisons performed against K-LoRA and the other half against ZipLoRA. In each trial, participants compared the results of two methods applied to the same concept–style pair and selected the preferred image according to the given criterion.

The results are summarized in Table~\ref{tab:user_study}, where each value denotes the percentage of participants preferring our method over the baseline. The study shows that while K-LoRA achieves slightly better concept preservation, our method is strongly preferred in terms of style transfer. ZipLoRA, in contrast, is outperformed by our method in both concept preservation and style transfer.

Q3 further shows that participants favor our results by a substantial margin over both K-LoRA and ZipLoRA.  Overall, the user study confirms that our approach produces images that better satisfy perceptual expectations of stylized concept generation.

%% file: sec/6_conclusion.tex
\section{Conclusion}
We introduce OrthoFuse, the first training-free method for orthogonal adapter merging. Our approach substantially improves style transfer fidelity while maintaining highly competitive concept preservation, striking a robust balance between the two.
By leveraging structured orthogonal parametrization and manifold-based geodesic approximations, our framework unites adapters tuned for different tasks into a single fused adapter without additional training. 
Extensive experiments in subject-driven generation tasks demonstrate that OrthoFuse outperforms existing fusion techniques, achieving superior style transfer while maintaining semantic consistency of the concept. 
Although some trade-offs between concept preservation and style fidelity remain, OrthoFuse establishes a robust, efficient, and principled foundation for multi-adapter fusion in diffusion models, enabling high-quality generation across diverse concept–style combinations.

\section*{Acknowledgments}

The work was supported by the grant for research centers in the field of AI provided by the Ministry of Economic Development of the Russian Federation in accordance with the agreement 000000C313925P4E0002 and the agreement with HSE University № 139-15-2025-009. 
The calculations were performed in part
through the computational resources of HPC facilities at HSE University~\citep{kostenetskiy2021hpc}.

%% file: sec/X_suppl.tex
\clearpage
\setcounter{page}{1}
\maketitlesupplementary
\appendix

\section{Smoothness}
\label{sec:smoothness_proof}

In this section, we prove that the set of \lpr~orthogonal matrices form a smooth manifold using a slightly different notation which is more convenient for these purposes. 
To do this, let us define additional objects we are going to use further.

For $i=1, \dots, m$ let \(B_i\) be a block-diagonal matrix with \(k_i\) orthogonal blocks of size \(b_i \times b_i \) and let \(P_i\) for $i=1, \dots, m-1$ be an arbitrary fixed permutation matrix.
Based on these notation, we define a set of orthogonal matrices $\mathcal{A}_m^{\text{orth}}$:
\begin{equation*}
\begin{split}
& \mathcal{A}^{\text{orth}}_m = \left\{ A \;\middle|\; A = B_{m} P_{m-1} \dots B_1 \right\}.
\end{split}
\end{equation*}
These set of matrices is a significantly more general set than \lpr~orthogonal matrices, considered in Section~\ref{sec:gs} due to the arbitrariness of $m$, the choice of every permutation matrix and the choice of the number of blocks in each $B_i$. 
For this set we provide our main technical contribution.

\begin{theorem} \label{theorem_smoothness}
    \( \mathcal{A}^{\text{orth}}_2 \) is a submanifold in \( \mathrm{O}(N) \) and in \( \mathrm{GL}_N(\mathbb{R}) \).
\end{theorem}

\begin{proof}
    To prove this, let us firstly define groups $\mathcal{B}_i$:
    \begin{equation}
    \mathcal{B}_i = \mathrm{O}(n_1) \times \dots \times \mathrm{O}(n_{k_i})    
    \end{equation}
    and their Cartesian product \( G = \mathcal{B}_2 \times \mathcal{B}_1 \). Note that \( G\) is a Lie group as it is a Cartesian product of two Lie groups. Additionally, \( G \) is compact because it is closed (being defined by the system of closed polynomial equations \( A A^\top = I \)) and bounded (each element of the Cartesian product is bounded by the square root of the matrix size in the Frobenius norm). 
    
    Note that the orthogonality condition implies that the transpose operation preserves the block-diagonal structure and orthogonality.
    It also gives us an opportunity to define Lie group action:
    \(G = \mathcal{B}_2 \times \mathcal{B}_1 \curvearrowright M\) where \(M\) is a manifold. 
    In our case we can choose \(M = \mathrm{O}(N)\) or \(M =  \mathrm{GL}_N(\mathbb{R})\) and take \(x = P_1 \in M\). The rule of action is defined as follows: 
    \begin{equation}
    g \cdot x: \mathcal{B}_2 \times \mathcal{B}_1 \times M \rightarrow M: B_2 x B_1^\top.
    \end{equation}
    Let us prove that it is really action by definition:
    \begin{itemize}
        \item \( g \cdot x \in M\) because orthogonal matrices are closed under multiplication.
        \item \( (g'g) \cdot x = B'_2 B_2 x (B'_1 B_1)^\top  = B'_2 B_2 x B^\top_1 (B'_1)^\top = g' \cdot (g \cdot x) .\)
        \item \( e \cdot x = I_2 \cdot x \cdot I_1^\top = x \cdot I_1^\top = x. \)
        \item The action is smooth as the composition of polynomial (hence smooth) matrix operations.
    \end{itemize}
    Thus, by definition (\cite{Vinberg_eng, Lee}), this is indeed a Lie group action.
    
   Using the fact that the orbit of a compact Lie group action on a manifold (either \(  \mathrm{GL}_N(\mathbb{R})\) or \( \mathrm{O}(N) \)) yields submanifold (see \citep[Theorem 2.3]{Vinberg_eng} or \citep[Corollary 21.6 \& Problem 21-17]{Lee}), we obtain the desired result.
    In addition to that, using the relation concerning group orbits we get: \( \text{Orb}_G(x) \cong G/ \text{Stab}_G (x), \) where
\begin{equation}
\begin{split}
\text{Orb}_G(x) = \{ g \cdot P_1 = B_2 P_1 B_1^\top \mid (B_2, B_1) \in \mathcal{B}_2 \times \mathcal{B}_1\}, 
\end{split}
\end{equation}
\begin{equation}
\begin{split}
\text{Stab}_G (P_1)& = \left\{(B_2, B_1) \right.  \left. \in \mathcal{B}_2 \times \mathcal{B}_1 \,\right|\, \\ 
& B_2 P_1 B_1^\top = P_1 \iff 
  \left. B_2 = P_1 B_1 P_1^\top \right\}
\end{split}
\end{equation}
and \(\cong\) denotes diffeomorphism (see the proof in, e.g., \cite{Lee}).
\end{proof}
\begin{remark}
We note that \cref{theorem_smoothness} is more general than our particular use case: diagonal blocks of \(B_i\) can be of different size and belong to \(\mathrm{U}\) (unitary matrices), \(\mathrm{SO}\) (special orthogonal matrices), or \(\mathrm{SU}\) (special unitary matrices). Moreover, the theorem statement is true for any matrix \(P_1\) taken from the manifold \(M\) (not necessarily a permutation matrix).
\end{remark}

Since we have shown that \(\mathcal{A}^{\text{orth}}_2\) is a submanifold, multiplying on the left (or right) by a fixed permutation matrix is a smooth diffeomorphism, and therefore the image \(\{\,P_L A P_R \mid A\in\mathcal{A}^{\text{orth}}_2\,\}\) is also a submanifold for any fixed permutation matrices \(P_L,P_R\).

\section{\texorpdfstring{\(\mathcal{G}\mathcal{S}\)}{GS}-orthogonal matrices merging}
\label{sec:merging_proof}

The main problem in \lpr-orthogonal matrices merging is that the orbit of the action is diffeomorphic to a homogeneous space. 
It means that the same matrix in manifold \(M\) can be obtained in several ways.
Here Perfect Shuffle permutation plays an important role: on the one hand, it provides good mixing, and on the other hand, its similarity transformation allows us to rather easily compute the stabilizer, i.e., obtain a description of the orbit. 
For the latter, we recall the result from \citep{dao2022monarch, GL}.
\begin{lemma}[\citep{dao2022monarch, GL}] \label{lem_congruent_PS}
Let \(P\) be a Perfect Shuffle permutation matrix. For any diagonal matrix \(D\) of size \( k_1 \times k_1\) and any matrix \(M\) of size \( b_1 \times b_1\), the following equation holds:
\begin{equation*}
    P (D \otimes M) P^\top = M \otimes D.
\end{equation*}      
\end{lemma}

From \cref{lem_congruent_PS}, it follows that when \( k_1 \geq b_2 \) the stabilizer is discrete. Consequently, by \citep[Theorem 21.17, Theorem 21.18]{Lee}, the discrete stabilizer implies that the orbit attains its maximal possible dimension:
\begin{equation}
\begin{split}
& \dim \text{Orb}_G(P_1) =  \dim \frac{G}{\text{Stab}_G (P_1)} = \\ ~&
= \dim G - \dim \text{Stab}_G (P_1) \leq \dim G - 0 = \dim G.
\end{split}
\end{equation}

The following lemma formally clarifies the desired result.

\begin{lemma}
    Stabilizer is discrete when \( k_1 \geq b_2\).
\end{lemma}

\begin{proof}
To prove that stabilizer is discrete, we need to solve 
\begin{equation}
B_2 = P_1 B_1 P_1^\top,
\end{equation}
where
\begin{equation}
    B_2 = \operatorname{diag}(B_2^{1}, B_2^{2}, \dots, B_2^{k_2}), \quad B_2^{i} \in \mathrm{O}(b_2)
\end{equation}
Let us denote \(C = P_1 B_1 P_1^\top.\) 

In terms of index notation, block-diagonal property for $B_2$ means that 
\[
(B_2)_{ij} \neq 0 \implies \left\lfloor \frac{i-1}{b_2} \right\rfloor = \left\lfloor \frac{j-1}{b_2} \right\rfloor
\]
Now consider an element $(B_2^l)_{i'j'}$ within a block of $B_2$, with local indices $i', j' \in \{1, \dots, b_2\}$ and global indices $i=(l-1)b_2+i'$, $j=(l-1)b_2+j'$. 
According to \cref{lem_congruent_PS}, to be non-zero, $(i, j)$-th of $B_2$ must satisfy $i \equiv j \pmod{k_1}$, and thus $i' \equiv j' \pmod{k_1}$.
In the case \( k_1 \ge b_2 \) the condition \( i' \equiv j' \pmod{k_1} \) for \( i', j' \in \{1, \dots, b_2\} \) holds only when \( i' = j' \) since $|i'-j'| < b_2 \le k_1$.
Thus, the matrix \( B_2^l \) can have nonzero elements only on the diagonal. Since \( B_2^l \) is orthogonal, it must be a diagonal orthogonal matrix. The diagonal elements of such a matrix are \( \pm 1 \). This set is finite and, as any stabilizer, forms a subgroup of \(G\) \cite{Lee}. The same result can be obtained for the case of blocks from \(\mathrm{SO}\).
\end{proof}
For the cases where \( k_1 \geq b_2 \) does not hold, the stabilizer (in general) is non-discrete.
\begin{remark}
   Note that the \textit{Perfect Shuffle} permutation is not the only matrix yielding maximum dimension under the condition \( k_1 \geq b_2 \). However, \textit{Perfect Shuffle} is optimal from the dense matrix formation perspective, as it minimizes the number of nonzero elements in the resulting matrix \citep{gorbunov2024group}. For these practical reasons, we employ \textit{Perfect Shuffle} in our structured representation.
\end{remark}

Now, knowing the structure of the manifold, we can find the geodesics. To do this, we need some additional theory from \cite{GHL, Lee}.

\begin{definition}[\cite{GHL, Lee}]
Let \( M \) and \( \tilde{M} \) be two manifolds. A map \( p : \tilde{M} \to M \) is a smooth covering map if:
\begin{itemize}
\item \( p \) is smooth and surjective,

\item for every point \( x \in M \), there exists a neighborhood \( U \) of \( x \) in \( M \) such that \( p^{-1}(U) \) is a disjoint union \( \bigcup_{i \in I} U_i \) of open subsets of \( \tilde{M} \), and for each \( i \in I \), the restriction \( p: U_i \to U \) is a diffeomorphism.
\end{itemize}
\end{definition}

\begin{definition}[\citep{GHL}]
Let \( (M, g) \) and \( (\tilde{M}, h) \) be two Riemannian manifolds. A map \( p: \tilde{M} \to M \) is a Riemannian covering map if:
\begin{itemize}
\item \( p \) is a smooth covering map,
\item \( p \) is a local isometry.
\end{itemize}
\end{definition}

While the formal definitions of isometry and local isometry are available in \cite{GHL}, they will not be essential for our subsequent development.

\begin{prop}[\citep{Lee}, Proposition 21.28] \label{prop_closed_subgroup}
    Every discrete subgroup of a Lie group is a closed Lie subgroup of dimension zero.
\end{prop}

\begin{prop}[\citep{Lee}, Proposition 21.34] \label{prop_connectness}
For each \(n \geq 1\), the Lie groups \(\mathrm{SO}(n),\mathrm{U}(n),\mathrm{SU}(n)\) are connected.
\end{prop}

\begin{prop}[\citep{Lee}, Theorem 21.29] \label{prop_covering_map}
     If \( G \) is a connected Lie group and \( \Gamma \subseteq G \) is a discrete subgroup, then \( G/\Gamma \) is a smooth manifold and the quotient map \( \pi: G \to G/\Gamma \) is a smooth normal covering map.
\end{prop}

\begin{prop}[\citep{GHL}, Proposition 2.18] \label{prop_riemannian_covering_map}
    Let \( p: N \to M \) be a smooth covering map. For any Riemannian metric \( g \) on \( M \), there exists a unique Riemannian metric \( h \) on \( N \) such that \( p \) is a Riemannian covering map.
\end{prop}

\begin{prop}[\citep{GHL}, Proposition 2.81] \label{geodesics_projection}
Let \( p: (N, h) \to (M, g) \) be a Riemannian covering map. The geodesics of \( (M, g) \) are the projections of the geodesics of \( (N, h) \), and the geodesics of \( (N, h) \) are the liftings of those of \( (M, g) \).
\end{prop}

\begin{prop}[\citep{Vinberg_eng, Lee}] \label{prop_commutative}
The following diagram is commutative:
\[
\begin{tikzcd}[row sep=large, column sep=large]
& G \arrow[dl, "\pi"'] \arrow[dr, "\varphi_{P_1}"] & \\
G/\text{Stab}_G(P_1) \arrow[rr, dashed, "f"'] & & \text{Orb}_G(P_1) \subseteq \mathrm{SO}(N)
\end{tikzcd}
\]
where \( \pi: G \to G/H \) is surjective map that sends each element \( g \in G \) to its corresponding coset \( gH \) (\(H\) denotes \(\text{Stab}_G(P_1) \)): \( \pi(g) = gH \). \( \varphi_{P_1}: G \to \text{Orb}_G(P_1) \, \) is the map defined by action \( \varphi_{P_1}(g) = g \cdot P_1 = B_2 P_1 B_1^\top \). Finally, \( f: G/H \to \text{Orb}_G(P_1) \, \) is the diffeomorphism defined by the rule: \( f(gH) := g \cdot P_1 \).
\end{prop}

\begin{proof}

Commutativity of the diagram means that for every \(g \in G\), the following identity holds:
\[
\varphi_{P_1}(g) = f(\pi(g)).
\]
Verification:
By the definition of \(\pi\): \(\pi(g) = gH\). By the definition of \(f\): \(f(\pi(g)) = f(gH) = g \cdot P_1\). By the definition of \(\varphi_{P_1}\): \(\varphi_{P_1}(g) = g \cdot P_1\).

Thus, \(\varphi_{P_1}(g) = f(\pi(g))\) for all \(g \in G\), which proves the commutativity of the diagram. The map \(f\) is a diffeomorphism and is well‑defined; see, for example, \cite{Lee}.
\end{proof}

Now let us combine these statements for the case \( k_1 \geq b_2 \). In this case, the stabilizer forms a subgroup. 

Therefore, using Propositions \ref{prop_closed_subgroup}, \ref{prop_connectness}, the fact that the Cartesian product of connected Lie groups is connected, and Proposition \ref{prop_covering_map}, we conclude that \( \pi: G \to G/H \) is a smooth covering map (since smooth normal covering map satisfies a stronger condition than a smooth covering map (see \cite{Lee})), where \(H = \text{Stab}_G(P_1)\).

Next, we show that the action map $\varphi_{P_1}$ is also a smooth covering map. To do this, it is sufficient to show, by Proposition \ref{prop_commutative}, that the composition of the diffeomorphism $f$ and the smooth covering map $\pi$ is a smooth covering map. Since $\pi$ is a smooth covering, for any $\bar{x} \in G/H$ there exists an evenly covered open neighborhood $U \subset G/H$ such that $\pi^{-1}(U) = \bigsqcup_{i \in I} V_i$, where each $V_i \subset G$ is open and the restriction $\pi|_{V_i}: V_i \to U$ is a diffeomorphism. Because $f$ is a diffeomorphism, $W = f(U)$ is an open neighborhood of $f(\bar{x})$ in $\text{Orb}_G(P_1)$. Then:
\[
\varphi_{P_1}^{-1}(W) = \pi^{-1}(f^{-1}(W)) = \pi^{-1}(U) = \bigsqcup_{i \in I} V_i.
\]
For each $i$, the restriction $\varphi_{P_1}|_{V_i} = f \circ \pi|_{V_i}$ is a diffeomorphism $V_i \to W$, as it is a composition of two diffeomorphisms. Thus, $\varphi_{P_1}$ satisfies the exact definition of a smooth covering map.

Finally, let \( g \) be the standard Riemannian metric on \( \mathrm{SO}(N) \) induced by the Frobenius inner product. As \( \text{Orb}_G(P_1) \) is an embedded submanifold, it inherits a Riemannian metric \( g_{\text{orb}} \).
According to Proposition~\ref{prop_riemannian_covering_map}, there exists a unique Riemannian metric \( h\) on \( G \) such that \( \varphi_{P_1} : (G, h) \to (\text{Orb}_G(P_1), g_{\text{orb}}) \) is a Riemannian covering map. By Proposition~\ref{geodesics_projection}, exact geodesics on the orbit are projections of exact geodesics in \( (G, h) \). Such metric need not necessarily be identical across blocks, nor necessarily coincide with the Frobenius norm. Nevertheless, we can perform a block-wise connection (uniform across blocks) using formula \eqref{eq:lmg} and empirically verify that the resulting curve exhibits nearly constant velocity in the Frobenius norm, consistent with the constant‑speed property of geodesics (see, \citep[Definition 2.77]{GHL}), thereby supporting its interpretation as a meaningful geodesic approximation with respect to the Frobenius norm. These findings are visually confirmed in Figure~\ref{fig:velocity_blockwise}.

Note that another natural approach is to follow a geodesic in the ambient space \(\mathrm{SO}(N)\); however, this is computationally expensive, inefficient, and generally does not stay within the \(\mathcal{GS}\) manifold. Nevertheless, we empirically demonstrate that the resulting block-wise curve closely approximates a geodesic in the ambient space \(\mathrm{SO}(N)\) (see Figure \ref{fig:two_approaches}).

\begin{remark}
Here and below we use the matrix logarithm, which is not always defined. In practice, however, we work with matrices whose diagonal blocks are sufficiently close to the identity, for which the logarithm is well defined (see Lemma~\ref{lem:log_existence} for a proof).
\end{remark}

\begin{figure*}[t!]
  \centering
    \includegraphics[scale=0.5]{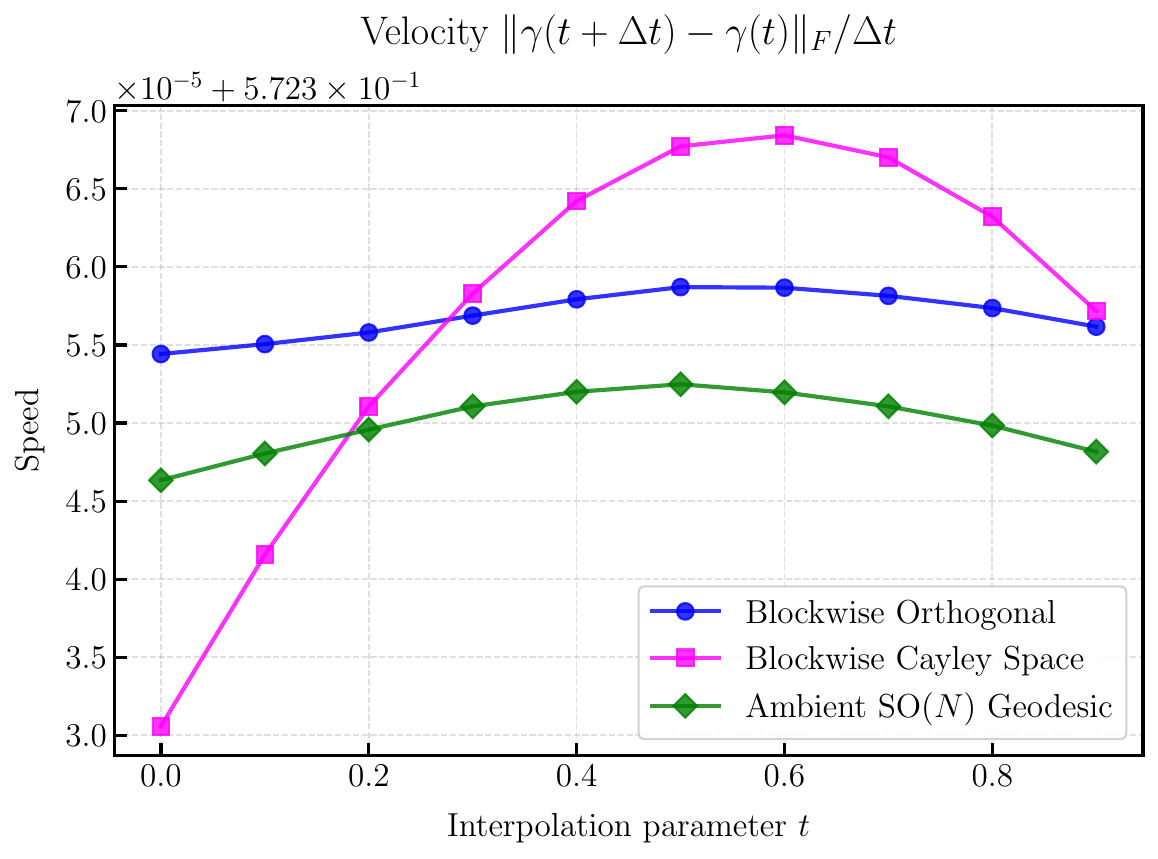}
    \caption{Velocity of geodesics measured in the Frobenius norm: comparison of the exact $\mathrm{SO}(N)$ geodesic, the block‑wise orthogonal approximation, and the block‑wise Cayley space approximation.}
\label{fig:velocity_blockwise}
\end{figure*}

\begin{figure*}[t!]
  \centering
    \includegraphics[width=1.0\linewidth]{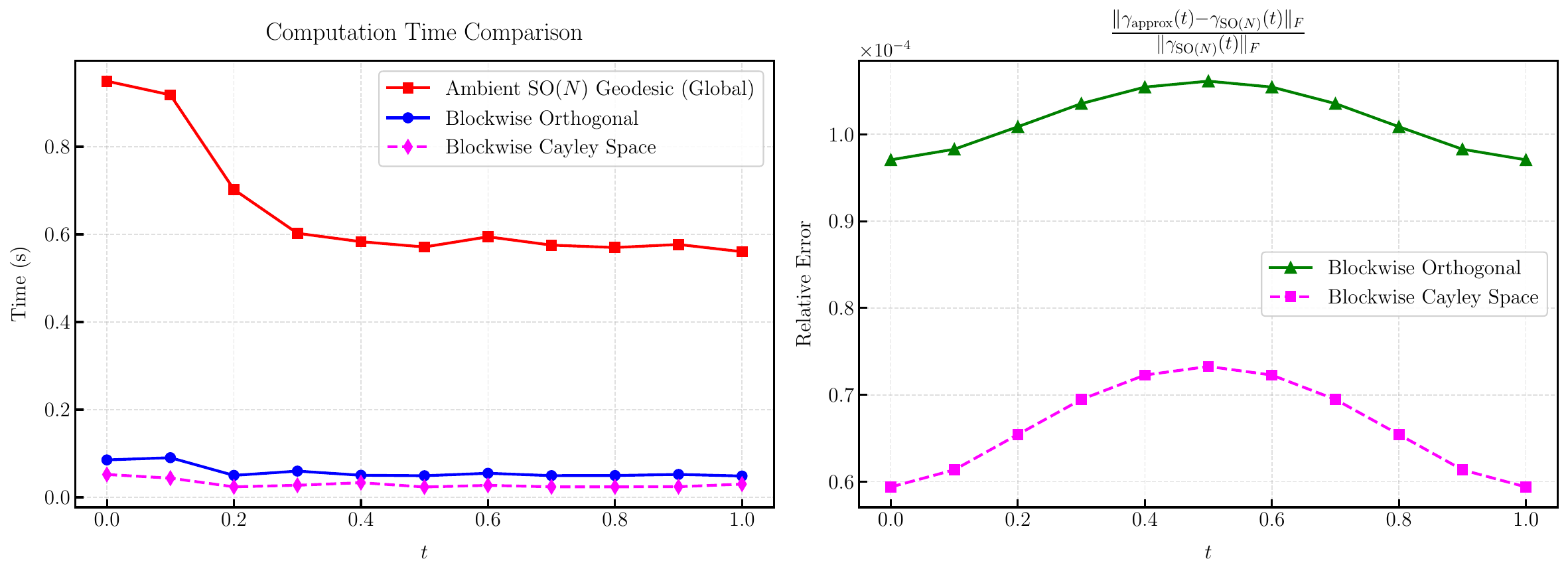}
    \caption{Relative error (in Frobenius norm) between the blockwise geodesic (geodesic projection from \(G\)) and the geodesic in \(M\), as well as the computation time of the two methods.}
\label{fig:two_approaches}
\end{figure*}

Observe that also that to obtain the group action we added transposition on the right, but we are given matrices in the form \(L_1 P R_1 = L_1 P (R_1^\top)^\top\) and \(L_2 P R_2 = L_2 P (R_2^\top)^\top\), and that's why we need to make additional transposition after connecting \(R_1^\top\) and \(R_2^\top\). 
Below we prove that this approach is equivalent to directly combining the matrix blocks without introducing additional transposition.

\begin{prop}
For \(A_1,A_2\in \mathrm{SO}(n)\) let  
\[
\gamma_{_{A_1,A_2}}(t)=A_1\exp\left(t\log(A_1^{\!\top\!}A_2)\right),\qquad t\in[0,1],
\]  
be the standard geodesic joining \(A_1\) to \(A_2\).  
Construct the analogous geodesic between the transposed matrices,  
\[
\gamma_{_{A_1^{\!\top\!},A_2^{\!\top\!}}}(t)=A_1^{\!\top\!}\exp\left(t\log(A_1A_2^{\!\top\!})\right),
\]  
and transpose the result. Then  
\[
\gamma_{_{A_1,A_2}}(t)=\left[\gamma_{_{A_1^{\!\top\!},A_2^{\!\top\!}}}(t)\right]^{\!\top\!}.
\]
\end{prop}

\begin{proof} 
\[
\left[\gamma_{_{A_1^{\!\top\!},A_2^{\!\top\!}}}(t)\right]^{\!\top\!}= \exp\left(t\log(A_1A_2^{\!\top\!})\right)^{\!\top\!}A_1.
\]
Now let's use \((\exp M)^{\!\top\!}=\exp(M^{\!\top\!})\) and \((\log M)^{\!\top\!}=\log(M^{\!\top\!})\) (this follows from the absolute convergence of the Taylor series):
\[
\left[\gamma_{_{A_1^{\!\top\!},A_2^{\!\top\!}}}(t)\right]^{\!\top\!} = \exp\!\left(t\log(A_2A_1^{\!\top\!})\right)A_1.
\]
Let \(\Omega=\log(A_1^{\!\top\!}A_2)\in \mathfrak{so}(n)\). Since
\[
\begin{split}
& A_1 \exp(\Omega) A_1^\top
= A_1 \left( \sum_{k=0}^{\infty} \frac{\Omega^k}{k!} \right) A_1^\top \\
= &\sum_{k=0}^{\infty} \frac{1}{k!} A_1 \Omega^k A_1^\top
= \sum_{k=0}^{\infty} \frac{1}{k!} (A_1 \Omega A_1^\top)^k \\
= &\exp(A_1 \Omega A_1^\top).
\end{split}
\]
A similar matrix equation is true for the logarithm function: for $A_1 \in \mathrm{SO}(N)$ we have
\begin{equation*}
    A_1\log(X)A_1^\top = \log(A_1 X A_1^\top).
\end{equation*}
Then
\[
A_2A_1^{\!\top\!}=A_1\exp(\Omega)A_1^{\!\top\!}=\exp(A_1\Omega A_1^{\!\top\!}).
\] 
Thus
\[
\log(A_2A_1^{\!\top\!})=A_1\Omega A_1^{\!\top\!}=A_1\log(A_1^{\!\top\!}A_2)A_1^{\!\top\!}.
\]
Finally, we obtain:
\begin{equation}
\begin{split}
\gamma_{A_1^{\top}, A_2^{\top}}(t) = & \exp\!\left(t\log(A_2A_1^{\!\top\!})\right)A_1 = \\ = & \exp\left(tA_1\log(A_1^{\!\top\!}A_2)A_1^{\!\top\!}\right)A_1 = \\ = & A_1\exp\left(t\log(A_1^{\!\top\!}A_2)\right)A_1^{\!\top\!}A_1 = \\ = & A_1\exp\!\left(t\log(A_1^{\!\top\!}A_2)\right)=\gamma_{_{A_1,A_2}}(t).
\end{split}
\end{equation}
\end{proof}

\section{Proof of Proposition~\ref{prop:log_orth}}\label{sec:proof_prop1}
\begin{proof}
    Utilizing eigendecomposition of $B(t)$ 
    \begin{equation}
        B(t) = U\Lambda U^*, \quad \Lambda = \diag(x_1 + iy_1, \dots, x_n + iy_n),
    \end{equation}
    satisfying $|x_i|^2 + |y_i|^2 = 1$. Since the eigenvalues of $B(t)$ lie on a unit sphere, we can express them as $e^{i\phi_i}$ for each $i = 1, \dots, n$. Then, for the principle branch of the logarithm, we have
    \begin{equation}\label{eq:prop3_lhs}
    \begin{split}
        \log(B(t)) = & ~U\log(\Lambda)U^* = \\ = & ~U \log(\diag(e^{i\phi_1}, \dots e^{i\phi_n}))U^* = \\ = & ~U \diag(i\phi_1, \dots i\phi_n)U^*.
    \end{split}
    \end{equation}
    Now let us consider $\frac{B(t) - B(t)^{\top}}{2}$. Using eigendecomposition, we have
    \begin{equation}\label{eq:prop3_rhs}
        \frac{B(t) - B(t)^{\top}}{2} = U \diag(iy_1, \dots iy_n) U^*.
    \end{equation}
    Subtracting from \eqref{eq:prop3_lhs} the final form of \eqref{eq:prop3_rhs}, we obtain
    \begin{equation}
    \begin{split}
        &\log(B(t)) - \frac{B(t) - B(t)^{\top}}{2} = \\ =&~ U\diag\left(i(\phi_1 - y_1), \dots i(\phi_n - y_n)\right)U^*
    \end{split}
    \end{equation}
    Using that $y_i = \sin \phi_i$, and assuming that $y_i$ is close to 0 (which is satisfied in practice), we can utilize Taylor expansion for $\sin$ function and finally obtain
    \begin{equation}
    \begin{split}
        &~U\diag\left(i(\phi_1 - y_1), \dots, i(\phi_n - y_n)\right)U^* = \\ 
        =&~ U\diag\left(\mathcal{O}(y_1^3), \dots, \mathcal{O}(y_n^3)\right)U^* = \\ =&~ \mathcal{O}\left(\|\diag(y_1, \dots, y_n)\|_2^3\right),
    \end{split}
    \end{equation}
    where the latter matrix belongs to $\mathcal{O}\left(\|B(t) - I\|_2^3\right)$. Indeed,
    
    \begin{equation}
        \begin{split}
            &\|B(t) - I\|_2 = \\ =&~ \|\diag(x_1 - 1 + iy_1, \dots, x_n - 1 + iy_n)\|_2^2 = \\ = &~ \|\diag(\sqrt{1-y_1^2}-1+iy_1, \dots, \\ &\ \sqrt{1-y_n^2}-1+iy_n)\|_2 \geqslant
            \\ \geqslant &~ 1/2\|\diag(y_1, \dots, y_n)\|_2
        \end{split}
    \end{equation}
    for small enough $\max_i|y_i|$, which completes the proof.
\end{proof}
\section{Proof of Proposition~\ref{prop:orthofuse}}\label{sec:proof_prop3}

\begin{proof}

Set skew-symmetric matrix $K(t) := \frac{B(t) - B(t)^\top}{2} $. From proposition \ref{prop:log_orth} we have:  
\[
\log B(t) = K(t) + \mathcal{O}\left(\|B(t) - I\|_2^3\right).
\]  
Hence,
\begin{equation}
\begin{split}
& B_{\text{Rotated}}(t) = \exp\left(\eta(t)\log B(t)\right) = \\ ~&
= \exp\left(\eta(t)K(t) + \mathcal{O}(\|B-I\|_2^3)\right).
\end{split}
\end{equation}

In last equation we used that \(\eta(t)\) has a closed form $1 + 4t(1-t)$, which is bounded on [0, 1].
Using smoothness of the matrix exponential we get:  
\[
B_{\text{Rotated}}(t) = \exp(\eta(t)K(t)) + \mathcal{O}\left(\|B(t) - I\|_2^3\right).
\]

From Proposition \ref{prop:exp_pade}, with $\eta(t)$ we obtain the following: 
\begin{equation}
\begin{split}
& \exp(\eta(t)K(t)) = \left(I - \frac{\eta(t)}{2}K(t)\right)^{-1}\left(I + \frac{\eta(t)}{2}K(t)\right) + \\ ~&
+ \mathcal{O}\left(\|\eta(t)K(t)\|_2^3\right).
\end{split}
\end{equation}

As $ \|\eta(t)K(t)\|_2 = \mathcal{O}(\|K(t)\|_2) = \mathcal{O}(\|B(t)-I\|_2) $, we obtain  
\[
\exp(\eta(t)K(t)) = B_{\text{OrthoFuse}}(t) + \mathcal{O}\left(\|B(t)-I\|_2^3\right).
\]
Combining all the results, we get
\[
B_{\text{OrthoFuse}}(t) = B_{\text{Rotated}}(t) + \mathcal{O}\left(\|B(t)-I\|_2^3\right).
\]

\end{proof}

\section{Additional Results on FLUX}
\label{sec:flux_results}

\begin{figure*}[t!]
  \centering
    \includegraphics[width=1.0\linewidth]
    {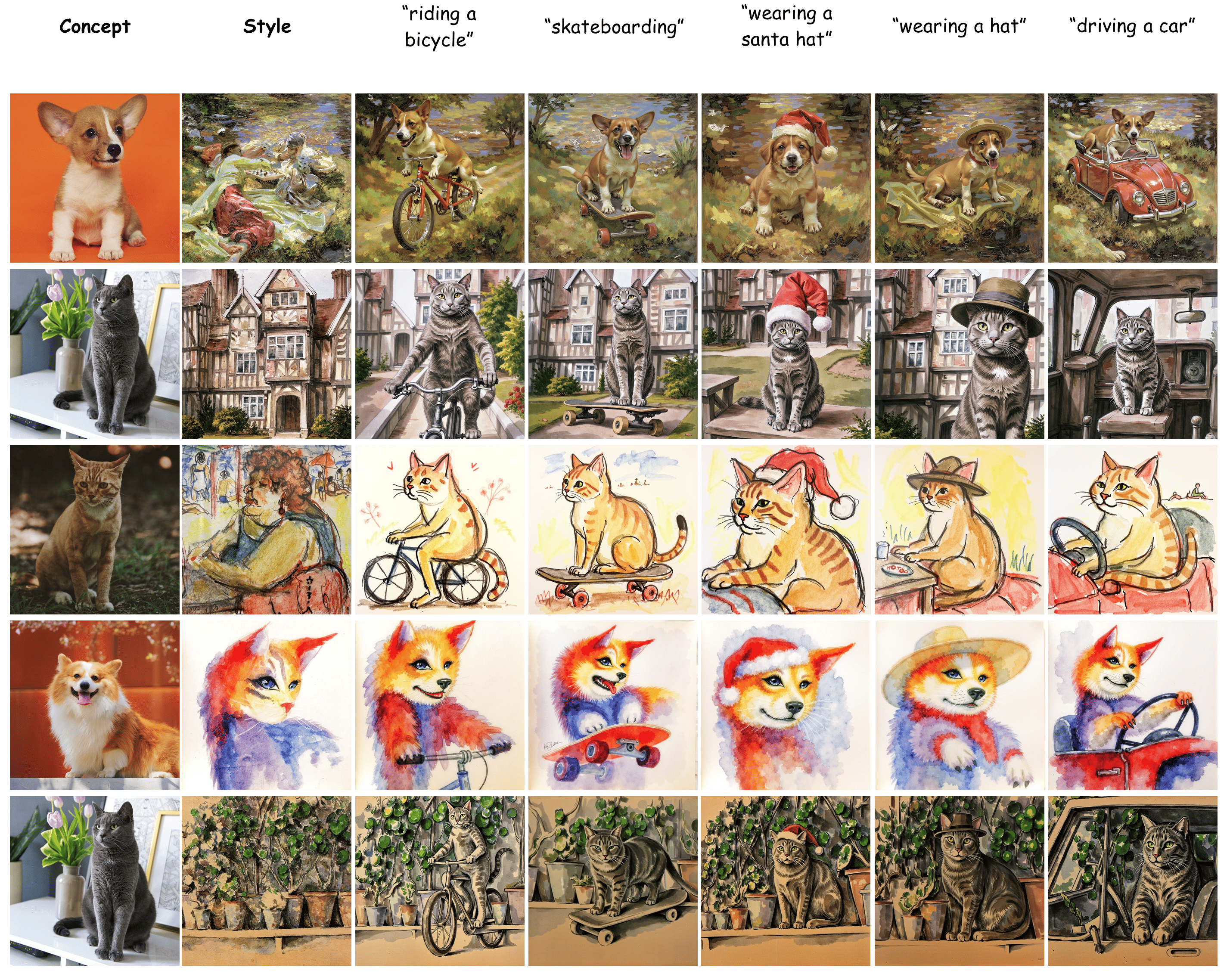}
    \caption{Qualitative results of \textbf{OrthoFuse} merging on the \textbf{FLUX} model. Each row shows generations produced from different prompts after merging a style adapter and a concept adapter. \textbf{OrthoFuse} maintains consistent concept preservation and style fidelity across prompts.}
\label{fig:flux_res}
\end{figure*}

We further evaluate OrthoFuse on the FLUX model to verify that our merging procedure can be applied to different model architectures. Figure \ref{fig:flux_res} shows qualitative generations obtained after merging a style adapter and a concept adapter in FLUX. Each row corresponds to a distinct style-concept pair, while each column shows outputs for different text prompts applied to the same merged adapters.

The results demonstrate that OrthoFuse produces stable and coherent merges across a diverse set of style–concept combinations. Even under varying prompts, the merged adapters consistently preserve the underlying concept while expressing the intended style.

\section{Necessity of Eigenvalue Rotation for High-Quality Merging}
\label{sec:blockwise_vs_rotation}

\begin{figure*}[t!]
  \centering
    \includegraphics[width=1.0\linewidth]
    {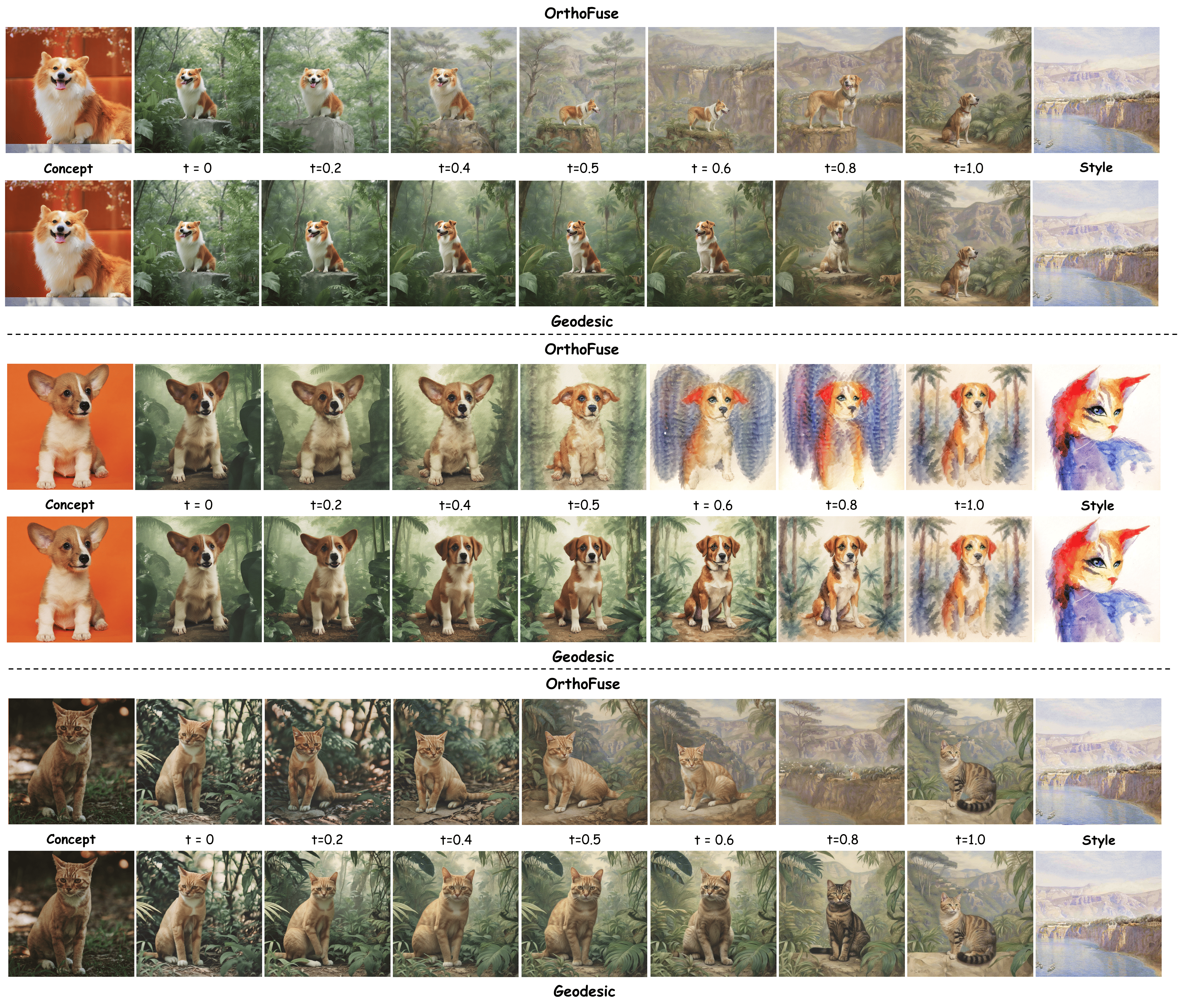}
    \caption{\textbf{Comparison of block-wise geodesic interpolation and OrthoFuse merging trajectories.} At $t=0.6$, OrthoFuse achieves near-ideal style transfer while preserving the target concept. All images were generated with the prompt: “A $<$concept$>$ $<$ superclass$>$ in jungle in $<$style$>$ style”.}
\label{fig:traj_comparison}
\end{figure*}

OrthoFuse method combines block-wise geodesic interpolation with \textit{spectra restoration}, which can be considered as a specific eigenvalue rotation along the unit sphere, preserving orthogonality (see Section \ref{sec:method} for more details). While block-wise geodesics provide a natural and accurate approximation of the real local minimizing geodesic in practice, we observe that fusing \lpr~orthogonal adapters with the help of block-wise geodesics only is insufficient for achieving high-quality semantic merging in diffusion models. Specifically, geodesic approximation via block-wise interpolation tends to drift away from the target concept and often fails to consistently align the style transformation across blocks.

Figure \ref{fig:traj_comparison} illustrates this effect by comparing the merging trajectories obtained with block-wise geodesic approximation and our full OrthoFuse procedure. At intermediate interpolation levels -- most clearly at $t=0.6$ -- the block-geodesic trajectory produces partially fused images where the transferred style is incomplete and the underlying concept begins to degrade. In contrast, OrthoFuse maintains both style fidelity and concept integrity, demonstrating that eigenvalue rotation plays a critical role in stabilizing the latent path and preventing semantic collapse.

All images in Figure~\ref{fig:traj_comparison} were generated with the prompt:
“A $<$concept$>$ $<$superclass$>$ in jungle in $<$style$>$ style.”

These results confirm that \textit{spectra restoration} operation is not merely an auxiliary refinement but an essential operation for producing coherent and high-quality merges.

\section{Additional Results on SDXL}
\label{sec:qual_app}
\begin{figure*}[t!]
  \centering
    \includegraphics[width=0.52\linewidth]
    {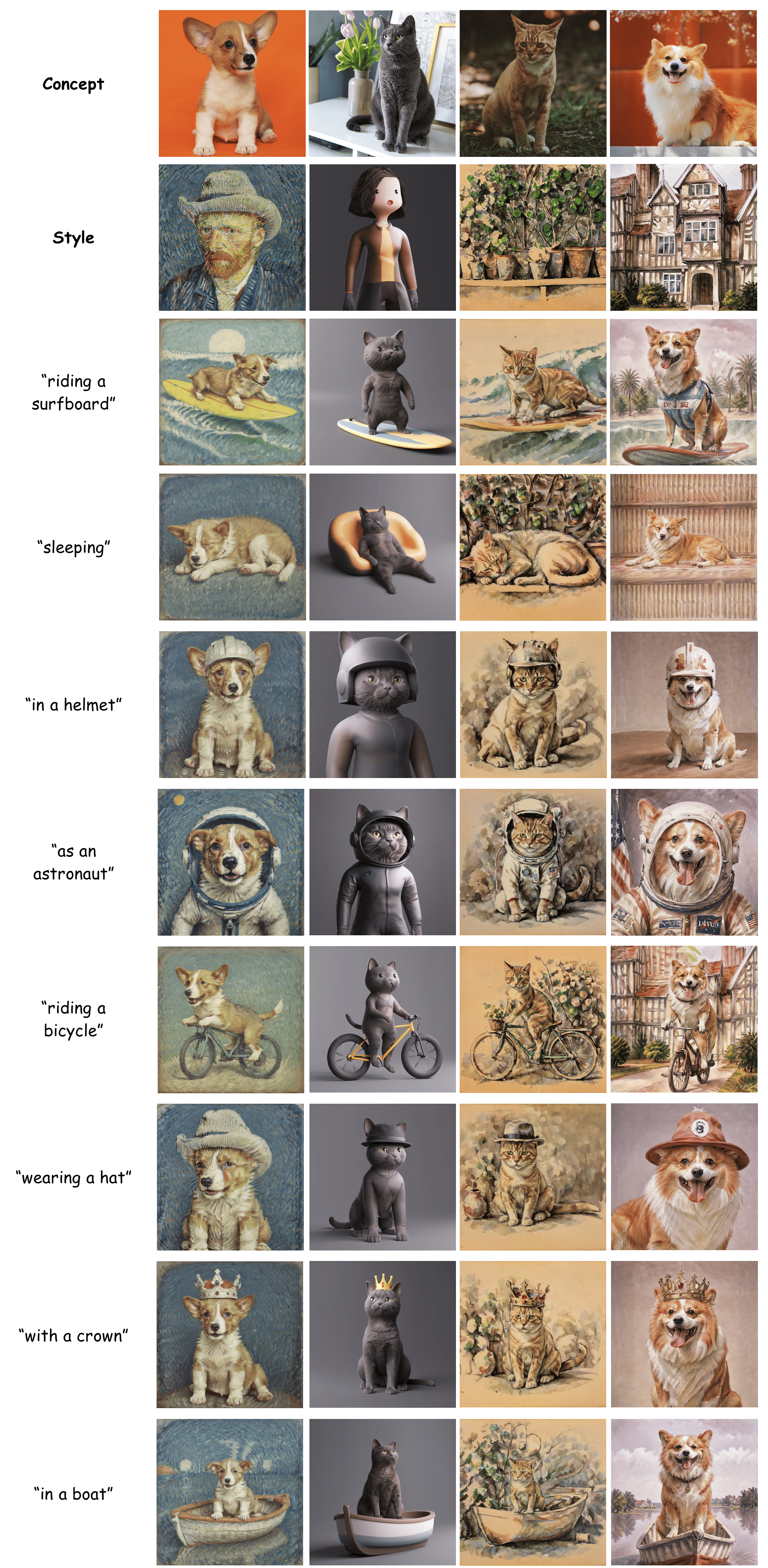}
    \caption{Qualitative results of \textbf{OrthoFuse} merging on \textbf{SDXL}. Rows correspond to different prompts; columns show generations obtained from different style--concept adapter pairs. OrthoFuse yields coherent style–concept merges across both prompts and adapter combinations.}
\label{fig:sdxl_prompts}
\end{figure*}
We additionally provide extended qualitative results on the SDXL backbone to complement the evaluations in the main paper. Figure \ref{fig:sdxl_prompts} presents generations obtained after merging a style adapter and a concept adapter within SDXL. In this visualization, each row corresponds to a different text prompt, while each column shows outputs for distinct style-concept adapter pairs applied to the same prompt.

Across all prompts and adapter configurations, OrthoFuse consistently achieves coherent style–concept integration, demonstrating strong concept preservation and stable expression of the intended style.

\section{Ablation study on other merging methods}
\subsection{Low-rank adapter merging}
\label{sec:low-rank_merge}

\begin{figure*}[t!]
  \centering
    \includegraphics[width=1.0\linewidth]
    {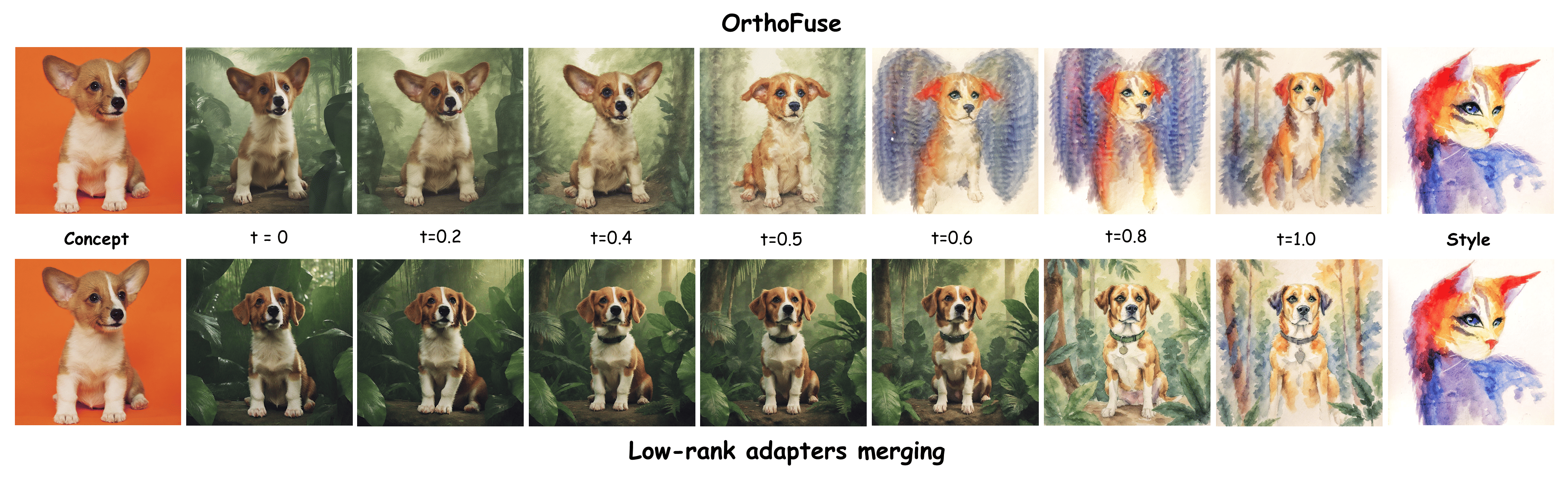}
    \caption{\textbf{Comparison of OrthoFuse (orthogonal adapters merging) with merging low-rank (LoRA) adapters.} Merging low-rank adapters results in noticeably weaker performance compared to OrthoFuse, even when both approaches are tuned to use approximately the same number of trainable parameters. We attribute this gap to the scale mismatch inherent to low-rank adapters, which makes their merging substantially more difficult. All images were generated with the prompt: “a $<$concept$>$ dog in the jungle in $<$style$>$ style”.}
\label{fig:lora_traj}
\end{figure*}

To highlight the applicability of orthogonal fine-tuning and extend the experimental scale, we provide a similar adapter merging experiment but for a fixed-rank manifold: we train low-rank adapters with LoRA for style and concept and try to merge them as fixed-rank manifold elements.
Assume that we aim to merge two low-rank matrices $X_{C}$ and $X_S$ which are low-rank weight updates for an arbitrary model layer:
\begin{equation}
    X_{C} = U_{C}V_{C}^{\top}, \quad X_{S} = U_{S}V_{S}^{\top},
\end{equation}
where $U_C, U_S \in \mathbb{R}^{n \times r}, V_C, V_S \in \mathbb{R}^{m \times r}$.
In the case of low-rank adapters, we aim to minimize the following objective: for $t \in [0, 1]$ we seek to optimize
\begin{equation}\label{eq:low_rank_manifold_merging}
t \cdot d^2_{\mathcal{M}_r}(X_{t}, X_S) + (1 - t) \cdot d^2_{\mathcal{M}_r}(X_{t}, X_C) \to \min_{\mathrm{rk}(X_{t}) = r},
\end{equation}
where $d_{\mathcal{M}_r}(\cdot, \cdot)$ denotes the distance along the manifold e.g. the shortest curve between two points along the manifold.
In our implementation, we replace the distance inside the manifold with the help of the Frobenius norm. 
Such a substitution is inspired by \citep{mataigne2025approximation}, which proposes certain theoretical guarantees on the closeness of such an approximation when the optimization is done with the change of the exact distance to its upper bound.
Having replaced the manifold distance with the Frobenius norm, we obtain the following minimization:
\begin{equation}
    t\|X_t - X_S\|_F^2 + (1 - t)\|X_t - X_C\|_F^2 \to \min_{\mathrm{rk}(X_{t}) = r}.
\end{equation}

This problem appears to be solved efficiently via ALS algorithm.
Indeed, for the current low-rank approximation of $X_t = U_t V_t^{\top}$ we are able to alternately update its skeleton factors with short recurrent formulas:
\begin{itemize}
    \item \textbf{V-step:} QR-decomposing $U = Q_UR$ and $X_{t} = Q_U\hat{V}$, we rewrite the task to the following one:
    \begin{equation}
        \|\hat{V}\|_F^2 - 2\langle \hat{V}^{\top}, Q_{U}(t X_{S} + (1 - t)X_{C})\rangle \to \min_{\hat{V}}
    \end{equation}
    Taking the gradient by $\hat{V}$, it gives us the update for $V$:
    \begin{equation}
    \begin{split}
        & 2\hat{V} - 2\left(Q_U^{\top}\left(tX_{C} + (1 - t)X_{S}\right)\right)^{\top} = 0 \Rightarrow \\ \Rightarrow & \hat{V} = (tX_{C} + (1 - t)X_{S})^{\top}Q_U
    \end{split}
    \end{equation}
    \item \textbf{U-step}: in a similar to $V$-step manner, one can obtain the following update rule for $U$: considering QR-decomposition of $V = Q_VR$ and $X_t = \hat{U}Q_V^{\top}$ we need to solve the same optimization problem
    \begin{equation}
        \|\hat{U}\|_F^2 - 2\langle \hat{U}^{\top}, Q_{V}(t X_{S} + (1 - t)X_{C})\rangle \to \min_{\hat{U}},
    \end{equation}
    from which we immediately obtain 
    \begin{equation}
        \hat{U} = (tX_{C} + (1 - t)X_{S})^{\top}Q_V.
    \end{equation}
\end{itemize}

It is worth mentioning that this problem can be easily generalized to the case of several low-rank adapters.
In this case, in $V$-step and $U$-step one need to replace the term $(tX_C - (1 - t)X_S)$ with the weighted sum of the corresponding low-rank adapters.

To validate this approach, in Figure \ref{fig:lora_traj} we report the empirical performance of the proposed merging method and compare it with OrthoFuse side by side. 
It can be observed that merging method applied to low-rank adapters performs worse than for \lpr~orthogonal adapters failing to preserve concept pattern and style fidelity.
To make the comparison fair, both methods were tuned using approximately the same number of parameters in corresponding parameter-efficient adapters.

\subsection{Orthogonal adapter merging via multiplication}

\begin{figure*}[t!]
  \centering
    \includegraphics[width=1.0\linewidth]
    {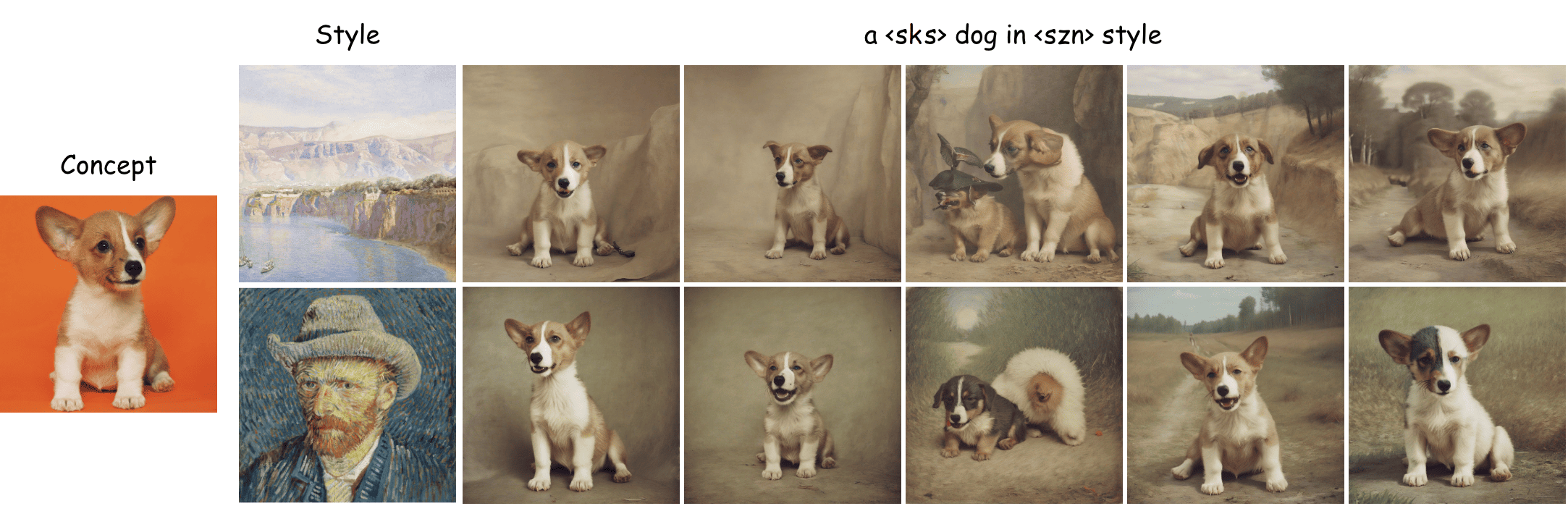}
    \caption{\textbf{Direct Merging via Multiplication.} The result of directly merging orthogonal adapters, accomplished through the multiplication of two \lpr~orthogonal matrices, exhibits limitations in style preservation, struggles to maintain color consistency, and has a negative impact on concept fidelity.}
\label{fig:direct_merging}
\end{figure*}

\begin{figure*}[ht!]
  \centering
   \includegraphics[width=1.0\linewidth]{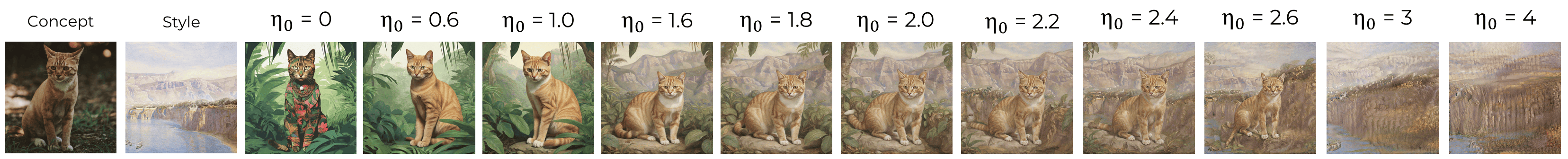}
   \caption{Ablation of the fusion parameter $\eta_0$.}
   \label{fig:eta_abl}
\end{figure*}

In order to additionally explain the motivation to take into account the geometry of the \lpr~orthogonal manifold, we try the evaluate the most trivial way to merge orthogonal adapters together by multiplying their weight updates. 
On Figure \ref{fig:direct_merging} we report images obtained by multiplying orthogonal weight updates of the concept and style respectively.
It can be seen that such an approach to fuse orthogonal adapters fails to preserve both style and concept patterns, which shows that for orthogonal adapters we need a more complicated approach which explicitly treats the structure of both orthogonal adapters.

\section{Ablation study of Fusion Parameters}
\label{sec:app_abl}

In this appendix, we present ablation studies analyzing the impact of two key fusion parameters, $\eta_0$ and $t$, on the performance of our proposed method.

\subsection{Ablation of $\eta_0$}

Figure~\ref{fig:eta_abl} illustrates the results of the ablation study for the fusion parameter $\eta_0$:
\begin{itemize}
    \item When $\eta_0 = 0$, the merged adapter collapses to the identity matrix, leading to no concept or style blending.
    \item At $\eta_0 = 1$, the model reduces to block-wise geodesic interpolation.
    \item The optimal performance is observed around ($\eta_0 \approx 2$), while larger values tend to degrade performance.
\end{itemize}

This analysis emphasizes the importance of appropriately selecting the fusion parameter $\eta_0$ to achieve the desired balance between concept and style.

\subsection{Ablation of $t$}

Figure~\ref{fig:sim_scores} presents the results of the ablation study for the fusion parameter $t$:
\begin{itemize}
    \item At $t = 0$, the method maximizes image similarity and minimizes style similarity.
    \item As $t$ increases, image similarity decreases while style similarity increases, with the maximum style similarity achieved at $t \approx 0.8$.
    \item Notably, the stylistic effects exhibited at \( t = 0.8 \) are stronger than those observed at \( t = 1 \). This is because, for \( t < 1 \), the eigenvalue transformation is applied, which can amplify the stylistic components. When \( t = 1 \), this transformation is disabled to recover the original style adapter from \eqref{eq:etafunc}, which can make the result at \( t = 0.8 \) appear stylistically stronger by comparison.
\end{itemize}

This behavior indicates the trade-off between image and style similarity, underscoring the significance of fine-tuning parameter $t$ for optimal performance.

\begin{figure}[h!]
  \centering
   \includegraphics[width=1.0\linewidth]{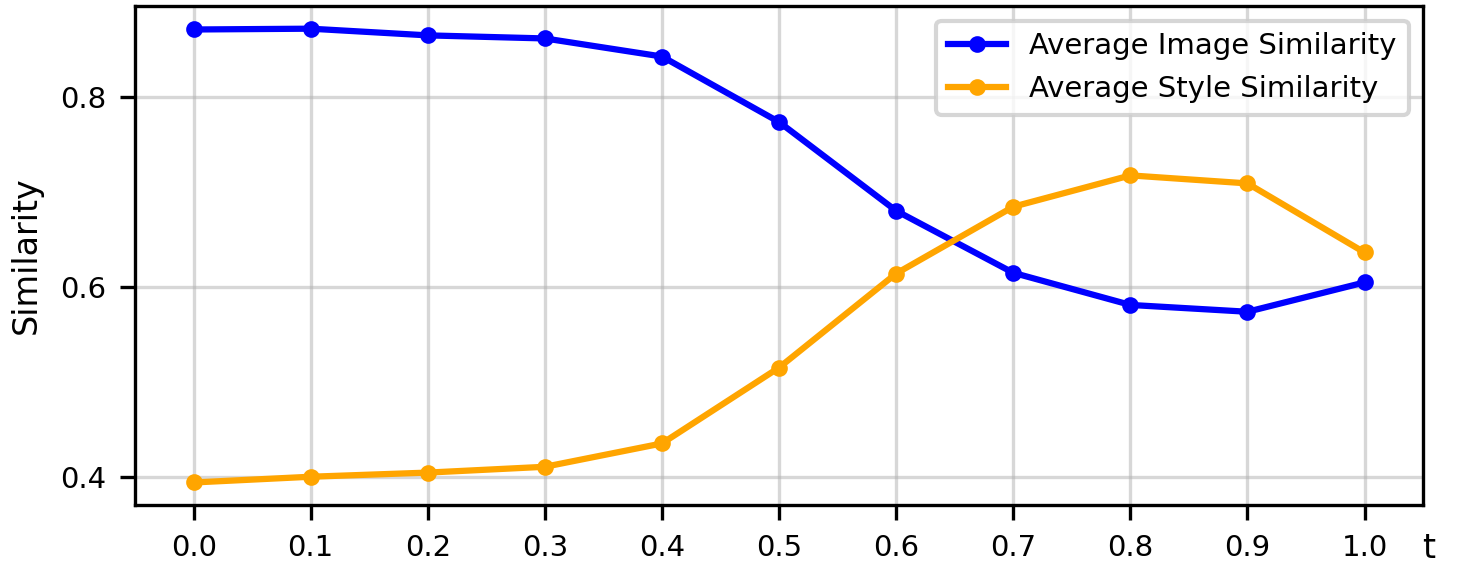}
   \caption{Ablation of the fusion parameter $t$.}
   \label{fig:sim_scores}
\end{figure}

\section{User Study}
\label{sec:app_user_study}

In this appendix, we provide further details about the user study conducted to evaluate the effectiveness of our proposed method compared to K-LoRA and ZipLoRA. To address the limitations of automatic metrics in assessing concept–style trade-offs, we designed a user study involving 65 participants. This study resulted in 1,460 pairwise comparisons across all images used in the evaluation.

Participants were asked to compare images generated by different methods applied to the same concept–style pair and respond to the following questions:

\textbf{Q1:} Which image better captures the features of the style?  
Evaluate whether the style is recognizable through visual characteristics (colors, textures, brush strokes, lines, etc.); can we say that the concept is genuinely represented in this style, rather than just slightly altered?

\textbf{Q2:} Which image better preserves the concept?  
Assess how well the original object (concept) is maintained; is it recognizable (shape, proportions, structure), and are important details retained? Please disregard any changes in pose.

\textbf{Q3:} Which method, in general, handled the task of style transfer to the concept better?    
Assess the overall result of the style transfer: 
\begin{itemize}
    \item Does it create the impression that the concept is naturally executed in the given style? 
    \item How well does the style harmonize with the object?  
\end{itemize}

If you do not see a difference regarding any question or are uncertain about your choice, please select "not sure."

\section{OrthoFuse Implementation Details}
\label{sec:app_code}

This section provides implementation details of the OrthoFuse algorithm used to construct the fused adapter $A(t)$ from independently trained concept and style adapters.

As described in Section~\ref{sec:gs}, both adapters are represented via block structures. 
In algorithms below we denote concept and style corresponding weight matrices $(D_C^{(i)}, D_S^{(i)})$. 
Importantly, $(D_C^{(i)}, D_S^{(i)})$ are weight matrices which are used to build skew-hermitian matrices for a subsequent Cayley transform application.

All merging operations are then performed independently on each corresponding pair of  blocks.

\subsection{Full OrthoFuse: Geodesic Block Interpolation}
The full OrthoFuse method performs interpolation along the geodesic in the orthogonal group for each block.

For every pair $\big(B_C^{(i)}, B_S^{(i)})$ we compute:

\begin{equation}
\widetilde{B}^{(i)}(t) = \text{Geodesic}\big(B_C^{(i)}, B_S^{(i)}, t\big).
\end{equation}

In practice, the geodesic is computed via:
\begin{enumerate}
    \item Conversion to skew-symmetric generators using the Cayley parameterization;
    \item Spectral decomposition of $B_S^\top B_C$;
    \item Logarithmic interpolation in the Lie algebra;
    \item Exponential map back to the orthogonal group
\end{enumerate}

The corresponding pseudocode for a single block is shown below.

\begin{algorithm}[h!tp]
	\caption{OrthoFuse merging}
\begin{algorithmic}[1]
	\Require $D_C, D_S$.
    \State $K_C = \frac{D_C-D_C^{\top}}{2};$ $K_S = \frac{D_S-D_S^{\top}}{2}$;
    \State $B_C = \texttt{torch.linalg.solve}((I-K_C)(I+K_C));$ $B_S = \texttt{torch.linalg.solve}((I-K_S)(I+K_S))$;
	\State $\Lambda, U = \texttt{torch.linalg.eig}(B_S^{\top}B_C)$;
	\State $\Lambda_{\text{log}} = \log(\Lambda).\texttt{imag} \cdot i$;
	\State $B_t = B_C\texttt{torch.linalg.matrix\_exp}(-t \cdot U \Lambda_{\text{log}} U^*).\texttt{real}$;\\
\Return $B_t$.
\end{algorithmic}
\end{algorithm}

The postprocessing step is defined as follows.

\begin{algorithm}[h!tp]
	\caption{OrthoFuse postprocess}
\begin{algorithmic}[1]
	\Require $B_t$.
    \State $ \eta = 1 + 4t(1 - t)$
    \State $Q = \eta B_t/2$;
    \State $Q^{skew} = \frac{Q-Q^{\top}}{2}$;
    \State $Q = \texttt{torch.linalg.solve}((I-Q^{skew})(I+Q^{skew}))$;\\
\Return $Q$.
\end{algorithmic}
\end{algorithm}

Overall, the full OrthoFuse procedure is defined as:
\texttt{OrthoFuse} =
\texttt{OrthoFuseMerging} + \texttt{OrthoFusePostprocess}.

\subsection{Accelerated OrthoFuse (Merge Inside Cayley Space)}

We also implement a computationally efficient approximation.

Instead of performing geodesic interpolation in 
$\mathrm{O}(k)$, we interpolate directly in the Cayley parameter space:

\begin{equation}
D^{(i)}_{merge} = tD_{C}^{(i)} + (1-t)D_{S}^{(i)}
\end{equation}

The merged skew-symmetric block is then mapped to the orthogonal group via the Cayley transform.

\begin{algorithm}[h!tp]
	\caption{OrthoFuse: merge inside Cayley space}
\begin{algorithmic}[1]
	\Require $D_C, D_S$.
    \State $D_{merge} = tD_C + (1-t)D_S$
    \State $K_{merge} = \frac{D_{merge}-D_{merge}^{\top}}{2}$;
    \State $B_t = \texttt{torch.linalg.solve}((I-K_{merge})(I+K_{merge}))$;\\
\Return $B_t$.
\end{algorithmic}
\end{algorithm}

\begin{figure*}[t!]
  \centering
    \includegraphics[width=1.0\linewidth]
    {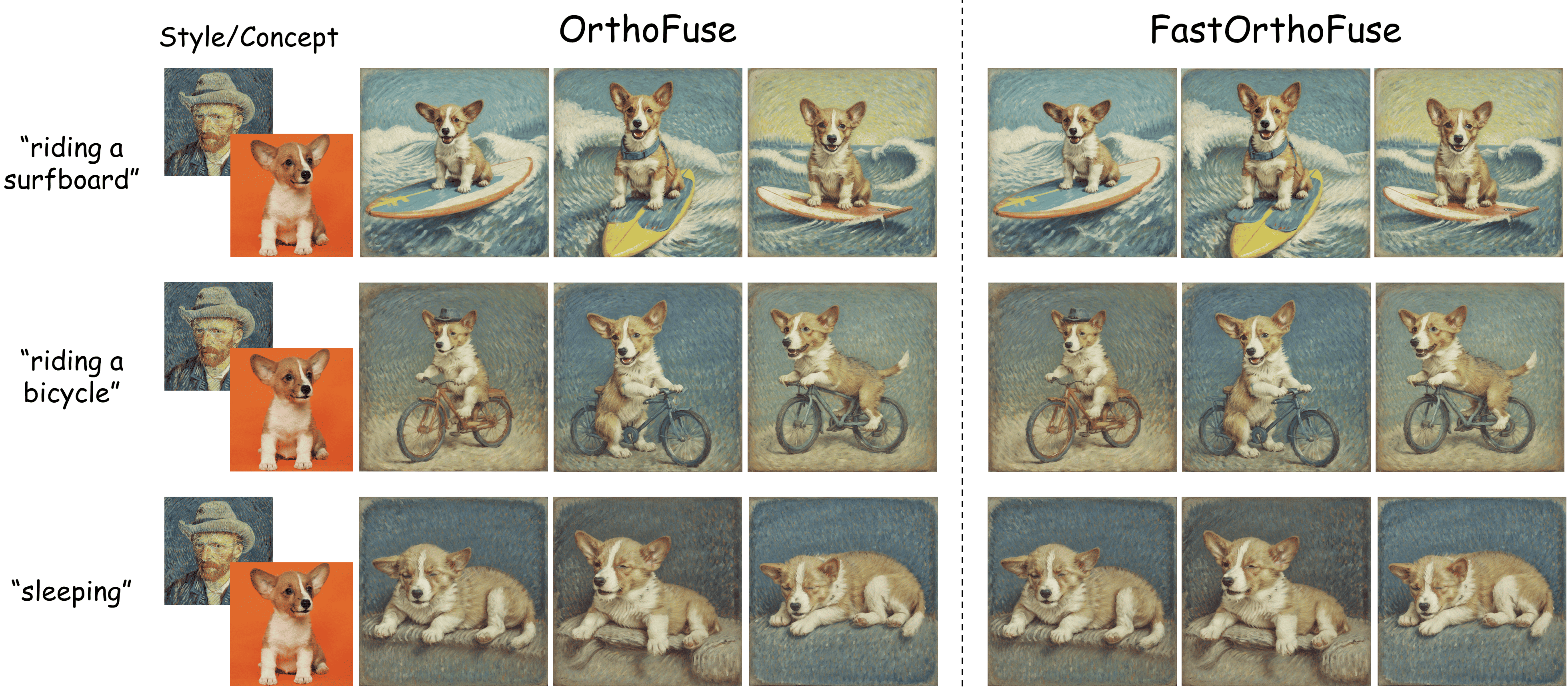}
    \caption{
    \textbf{Comparative analysis of OrthoFuse and its accelerated version}. Using identical concept and style adapters and the same fusion parameter $t$, both methods produce visually indistinguishable results. The accelerated version removes the eigendecomposition step while preserving identity and style fidelity. We note, however, that the two methods are not strictly identical; for example, in the first row, the dog's right paw in the OrthoFuse result is not fully placed on the surfboard, whereas in FastOrthoFuse it is.    
    }
\label{fig:fast_vs_full}
\end{figure*}

The full accelerated pipeline is therefore:

\texttt{FastOrthoFuse} =
\texttt{MergeInsideCayleySpace} + \texttt{OrthoFusePostprocess}.

\subsection{Theoretical Justification of Accelerated OrthoFuse}

First, we show that the matrix logarithm is correctly defined.

\begin{lemma} \label{lem:log_existence}
Assume orthogonal matrices $B_S, B_C \in \mathrm{SO}(n)$ satisfy $\|B_S-I\|_2 \le \varepsilon < 1$ and $\|B_C-I\|_2 \le \varepsilon < 1$. Then the matrix logarithm $\log(B_S^\top B_C)$ is well-defined.
\end{lemma}
\begin{proof}
Using the triangle inequality and submultiplicativity, we bound the spectral norm of the difference:
$$\begin{aligned} \|B_S^\top B_C-I\|_2 &\le \|B_S^\top-I\|_2+\|B_C-I\|_2  \le 2\varepsilon. \end{aligned}$$
Since $B_S^\top B_C$ is orthogonal, its eigenvalues lie on the unit circle. If $-1$ were an eigenvalue, the distance to the identity $I$ would be at least $|-1-1|=2$. Hence, whenever $2\varepsilon<2$, the matrix $B_S^\top B_C$ has no eigenvalues equal to $-1$.
\end{proof}

Now we make use of the following auxiliary lemma.
\begin{lemma}\label{thm:bch}
For sufficiently small matrices $X$ and $Y$ with $\|X\|_2, \|Y\|_2 = \mathcal{O}(\varepsilon)$, the matrix logarithm of their exponential product is given by:
$$\log(\exp(X)\exp(Y))=X+Y+\frac{1}{2}[X,Y]+\mathcal{O}(\varepsilon^3)$$
\end{lemma}
\begin{proof}
This follows from well-known (see, for example, \cite{Lee}) Baker-Campbell-Hausdorff formula. By submultiplicativity, the norm of the commutator satisfies $\|[X,Y]\|_2 \le 2\|X\|_2\|Y\|_2 = \mathcal{O}(\varepsilon^2)$. Consequently, nested commutators such as $[X,[X,Y]]$ inherently possess a higher order of smallness $\mathcal{O}(\varepsilon^3)$. Similarly arguing by induction, it is straightforward to establish that each subsequent nesting of the commutator increases the order of smallness.
\end{proof}

Finally, we prove that geodesics can be approximated by a connection in the space of skew‑symmetric matrices.
\begin{prop} \label{prop:cayley_geodesic_approx}
Let $B_C, B_S \in \mathrm{SO}(n)$ be orthogonal matrices parameterized by skew-symmetric matrices $D_C, D_S \in \mathfrak{so}(n)$ via the Cayley transform: $B_C=\mathrm{Cayley}(D_C)$ and $B_S=\mathrm{Cayley}(D_S)$, where $\|D_C\|_2, \|D_S\|_2 = \mathcal{O}(\varepsilon)$. The Cayley transform of their linearly interpolated generators approximates the exact Riemannian geodesic $B(t)=B_C\exp(-t\log(B_S^\top B_C))$ up to a third-order error:
$$B(t)=\mathrm{Cayley}((1-t)D_C+tD_S)+\mathcal{O}(\varepsilon^3)$$
\end{prop}

\begin{proof}
Recall that the Cayley transform matches the matrix exponential up to the second order: $\mathrm{Cayley}(K)=\exp(K)+\mathcal{O}(\varepsilon^3)$. We rewrite our endpoints as $B_C=\exp(D_C)+\mathcal{O}(\varepsilon^3)$ and $B_S^\top=\mathrm{Cayley}(-D_S)=\exp(-D_S)+\mathcal{O}(\varepsilon^3)$. Applying Lemma \ref{thm:bch}, we approximate the logarithm term:
$$\begin{aligned} \log(B_S^\top B_C) &= \log(\exp(-D_S)\exp(D_C)) \\ &= D_C-D_S-\frac{1}{2}[D_S,D_C]+\mathcal{O}(\varepsilon^3) \end{aligned}$$
Substituting this into the geodesic equation yields $B(t)=\exp(D_C)\exp(V)+\mathcal{O}(\varepsilon^3)$, where the exponent is defined as $V=-t(D_C-D_S)+\frac{t}{2}[D_S,D_C]$. We apply Lemma \ref{thm:bch} a second time to combine these into a single exponential $\exp(Z)$, where $Z=D_C+V+\frac{1}{2}[D_C,V]$. 

Given that $\frac{1}{2}[D_C, V] = \frac{1}{2}[D_C, -t(D_C - D_S)] + \mathcal{O}(\varepsilon^3)$, we expand $Z$:
$$\begin{aligned} Z &= D_C-t(D_C-D_S)+\frac{t}{2}[D_S,D_C] \\ &\quad +\frac{1}{2}[D_C,-t(D_C-D_S)] + \mathcal{O}(\varepsilon^3) \end{aligned}$$
Due to the anti-symmetry of the Lie bracket, the second-order commutators perfectly cancel each other out:
$$\frac{1}{2}[D_C,-t(D_C-D_S)]=\frac{t}{2}[D_C,D_S]=-\frac{t}{2}[D_S,D_C]$$
This exact cancellation reduces the exponent to $Z=(1-t)D_C+tD_S$. Therefore, the geodesic simplifies to $B(t)=\exp((1-t)D_C+tD_S)+\mathcal{O}(\varepsilon^3)$. Applying the Padé equivalence $\exp(Z)=\mathrm{Cayley}(Z)+\mathcal{O}(\varepsilon^3)$ once more concludes the proof.
\end{proof}

Another way to think about this is as follows. After training the adapters using skew‑symmetric matrices, we obtain their final representations and then apply the Cayley transform (a retraction) to obtain orthogonal matrices. Thus, performing linear interpolation in the skew‑symmetric parameter space corresponds to mixing weights --- a concept reminiscent of task arithmetic (see, e.g., \cite{ilharco2023editing}).

\subsection{Computational Considerations}

The computational bottleneck of the full OrthoFuse algorithm is the eigendecomposition step:

\begin{equation}
\Lambda, U = \texttt{torch.linalg.eigh}(Q_S^\top Q_C).
\end{equation}

For the matrix sizes used in our adapters, this operation accounts for approximately \textbf{90\% of the total merging time}. In PyTorch, this routine is not efficiently parallelized in our setting and dominates the wall-clock runtime.

The accelerated variant completely removes the eigendecomposition. All remaining operations (matrix multiplications, linear solves, and matrix exponentials) are efficiently parallelized, leading to a substantial speedup. In practice, the fast version performs adapter merging in under one second while nonaccelerated version works in 90 seconds. 

\subsection{Empirical Observation}

Despite the simplification, the accelerated version produces images that are visually almost indistinguishable from those obtained using the full geodesic interpolation (see Figure~\ref{fig:fast_vs_full}). In our experiments, we observe no meaningful degradation in identity preservation or style transfer quality, while achieving a significant reduction in computational cost.

\section{Limitations}

While OrthoFuse is training-free, it requires adapters to be in \( \mathcal{GS} \)-orthogonal form. Most community adapters are standard LoRAs; applying our method directly to them would need a projection, which risks losing the information encoded in the original LoRA weights. Extending our fusion to LoRA weights is a promising future direction.